\documentclass[nohyperref]{article}

\usepackage{microtype}
\usepackage{graphicx}
\usepackage{subfigure}
\usepackage{booktabs} %

\usepackage{hyperref}

\usepackage[accepted]{icml2022mod}

\usepackage{amsmath}
\usepackage{amssymb}
\usepackage{mathtools}
\usepackage{amsthm}

\usepackage[capitalize,noabbrev]{cleveref}

\theoremstyle{plain}
\newtheorem{theorem}{Theorem}[section]

\newtheorem{lemma}[theorem]{Lemma}

\theoremstyle{definition}

\newtheorem{assumption}[theorem]{Assumption}
\theoremstyle{remark}
\newtheorem{remark}[theorem]{Remark}
\newtheorem{extended_theorem}[theorem]{Extended Theorem}

\usepackage[textsize=tiny]{todonotes}

\usepackage{amssymb,amsmath,amsthm,dsfont}

\providecommand{\lin}[1]{\ensuremath{\left\langle #1 \right\rangle}}
\providecommand{\abs}[1]{\left\lvert#1\right\rvert}
\providecommand{\norm}[1]{\left\lVert#1\right\rVert}

  \providecommand{\R}{\mathbb{R}} 
  
  \DeclareMathOperator{\E}{{\mathbb E}}
  \providecommand{\E}[1]{{\mathbb E}\left.#1\right. }     

  \DeclareMathOperator*{\argmin}{arg\,min}

  \providecommand{\0}{\mathbf{0}}

  \providecommand{\rr}{\mathbf{r}}
  
  \providecommand{\tt}{\mathbf{t}}
  \providecommand{\uu}{\mathbf{u}}
  \providecommand{\vv}{\mathbf{v}}
  \providecommand{\ww}{\mathbf{w}}
  \providecommand{\xx}{\mathbf{x}}
  \providecommand{\yy}{\mathbf{y}}
  \providecommand{\zz}{\mathbf{z}}
  
  \providecommand{\mI}{\mathbf{I}}

  \providecommand{\cD}{\mathcal{D}}

  \providecommand{\cG}{\mathcal{G}}

  \providecommand{\cN}{\mathcal{N}}
  \providecommand{\cO}{\mathcal{O}}

  \providecommand{\cU}{\mathcal{U}}
  
  \providecommand{\cX}{\mathcal{X}}
  
  \providecommand{\cW}{\mathcal{W}}

  \providecommand{\bxi}{\boldsymbol{\xi}}

\providecommand{\comment}[3]{\todo[caption={},color=#3!20]{\textbf{#1: }#2}}%
\providecommand{\inlinecomment}[3]{%
  {\color{#1}#2: #3}}%
\newcommand\commenter[2]%
{%
  \expandafter\newcommand\csname i#1\endcsname[1]{\inlinecomment{#2}{#1}{##1}}
  \expandafter\newcommand\csname #1\endcsname[1]{\comment{#1}{##1}{#2}}
}

\usepackage{url}

\usepackage[T1]{fontenc}
\commenter{seb}{red} %
\commenter{harsh}{blue}
\usepackage{enumitem}
\let\qedhere\relax %
\usepackage{mathtools}

\usepackage{subfigure}
\usepackage{makecell}
\usepackage{multirow}
\usepackage{pifont}
\newcommand{\circledone}{\ding{172}}
\newcommand{\circledtwo}{\ding{173}}
\usepackage{booktabs}

\newcommand\blfootnote[1]{%
  \begingroup
  \renewcommand\thefootnote{}\footnote{#1}%
  \addtocounter{footnote}{-1}%
  \endgroup
}

\usepackage[font=small]{caption}

\let\cite\citep 
\icmltitlerunning{Tackling benign nonconvexity with smoothing and stochastic gradients}

\begin{document}

\twocolumn[
\icmltitle{Tackling benign nonconvexity with smoothing and stochastic gradients}
\icmlsetsymbol{equal}{*}

\begin{icmlauthorlist}
\icmlauthor{Harsh Vardhan}{ucsd}
\icmlauthor{Sebastian U. Stich}{cispa}
\end{icmlauthorlist}

\icmlaffiliation{ucsd}{University of California, San Diego, \textless{}hharshvardhan@ucsd.edu\textgreater{}.}
\icmlaffiliation{cispa}{CISPA Helmholz Center for Information Security, \textless{}stich@cispa.de\textgreater{}}

\icmlkeywords{Machine Learning, Smoothing}

\vskip 0.3in
]

\printAffiliationsAndNotice{}  %

\begin{abstract}
Non-convex optimization problems are ubiquitous in machine learning, especially in Deep Learning.
While such complex problems can often be successfully optimized in practice by using stochastic gradient descent (SGD), theoretical analysis cannot adequately explain this success. In particular, the standard analyses do not show global convergence of SGD on non-convex functions, and instead show convergence to stationary points (which can also be local minima or saddle points).\\%
We identify a broad class of nonconvex functions for which we can show that perturbed SGD (gradient descent perturbed by stochastic noise---covering SGD as a special case) converges to a global minimum (or a neighborhood thereof), in contrast to gradient descent without noise that can get stuck in local minima far from a global solution.
For example, on non-convex functions that are relatively close to a convex-like (strongly convex or P\L) function we show that SGD can converge linearly to a global optimum.
\end{abstract}

\section{Introduction}
\blfootnote{An earlier version of this paper~\cite{vardhanescaping} has been accepted at NeuRIPS 2021 workshop Optimization in Machine Learning}
Non-convex optimization problems are ubiquitous in deep learning and computer vision~\cite{Bottou2018:book}. The training of a neural network amounts to minimizing a non-convex  loss function $f\colon \R^d \to \R$,  
\begin{align} 
 f^\star = \min_{\xx \in \R^d} \bigl[ f(\xx) = \E_{\bxi \sim \cD} f(\xx,\bxi) \bigr] \,, \label{eq:problem}
\end{align}
where stochastic gradients $\nabla f(\xx,\bxi)$ can be evaluated on samples $\bxi \sim \cD$ of the data distribution (this formulation covers both the online setting or empirical risk minimization on a finite set of samples).
Stochastic gradient descent methods, like SGD~\cite{Robbins:1951sgd} or ADAM~\cite{kingma2014adam}, are core components for training neural networks. 
In addition to their simplicity, and almost universal applicability, the solutions obtained by stochastic methods often
generalize remarkably well~\citep[see e.g.][]{Keskar2017}. \looseness=-1 

The analysis of SGD-type methods for smooth objective functions is well understood:\ to find an $\epsilon$-approximate stationary point, i.e.\ $\norm{\nabla f(\xx)} \leq \epsilon$, SGD needs $\cO\bigl(\epsilon^{-4}\bigr)$ gradient evaluations~\cite{Ghadimi2013}, SGD with recursive momentum requires $\cO\bigl(\epsilon^{-3}\bigr)$ gradient evaluations~\cite{cutkosky2019momentum}, which is optimal~\cite{arjevani2019lower}, and in the deterministic setting, gradient descent converges in  $\cO\bigl(\epsilon^{-2}\bigr)$ gradient evaluations~\cite{Nesterov2004:book}.
Still, in practice it is often possible to find approximate stationary points---and even approximate
global minimizers---of nonconvex functions faster than these complexity bounds suggest. This performance
gap stems from the fairly weak smoothness assumption underpinning these generic bounds. 
However, functions minimized in practice often admit significantly more structure, even
if they are not convex.

An active line of research has started to  characterize classes of functions for which \emph{gradient type} methods work well, i.e.\ discrete methods that track the gradient flow.
For instance, 
\citet{Ge2016} show that matrix completion exhibits ``convexity-like'' properties, i.e.\ that all local minimizers are global.
In more abstract settings, 
\citet{Polak1963:pl,Lojasiewicz1963:pl} study gradient dominated functions,
\citet{necoara_linear_2016} star-convex functions and \citet{hinder_near-optimal_2019} investigate quasar-convex functions.
All these function classes have in common that the gradient flow converges to a unique minima.
To show convergence of stochastic methods, it therefore suffices to control the stochastic noise, i.e.\ to show that the steps taken by the algorithm follow sufficiently close the true gradient direction  $\E_{\bxi} \nabla f(\xx,\bxi)=\nabla f(\xx)$.
This can for instance be achieved by averaging techniques~\cite{Moulines2011:nonasymptotic,stich_unified_2019}, decreasing stepsizes~\cite{lacoste-julien_simpler_2012} or variance reduction~\cite{Johnson2013,Zhang2013,Mahdavi2013,Wang2013}.

However, these arguments cannot explain the success of SGD on functions with \emph{multiple local minima} on which the gradient flow can get stuck in local minima that are far from a global optimal solution.
To improve our understanding of the convergence of SGD on such functions we also need to consider the effect of stochastic noise (and algorithmic randomness).
Stochastic noise has been observed to have many beneficial effects in non-convex optimization:\
For instance, it has been proven that stochastic noise can allow SGD to escape saddle points \cite{Ge2015:escaping,jin_how_2017,daneshmand_escaping_2018},
and under certain conditions noise allows SGD to escape local minima~\cite{hazan_graduated_2015,kleinberg_alternative_2018}.
 In DL, it has been observed that artificially injected noise can lead to improved generalization~\cite{neelakantan2015adding,chaudhari2017entropy,Plappert2017}, in particular in the context of large batch training~\cite{wen2018smoothout,Haruki2019,Lin2020}.

In this work, 
we characterize a new class of non-convex functions for which stochastic gradient methods can provably escape certain types of local minima.
In particular, we characterize  non-convex functions on which stochastic methods converge \emph{linearly} to a global solution (in contrast, only sublinear convergence rates to local minima are known on general non-convex functions, \citealp{pmlr-v99-fang19a,Li2019:ssrgd}).\footnote{Concretely, $\cO(1/\epsilon^{3.5})$ complexity to find an $\sqrt{\epsilon}$-approximate local minima with $\norm{\nabla f(\xx)} \leq \epsilon$ \cite{pmlr-v99-fang19a,Li2019:ssrgd}.}

The class of structured functions that we study in this work, are functions $f$ that have a hidden composite structure. This structure is in general \emph{unknown} to the algorithm (the algorithm can only query $\nabla f(\xx,\bxi)$ (such as SGD) and \emph{does not have} access to $g$ or $h$ separately).
Concretely, we assume that $f$ is the composition of two components $g,h \colon \R^d \to \R$:
\begin{align}
 f(\xx) = g(\xx) + h(\xx)\,.
\label{eq:structure}
\end{align}
As an intuitive example, suppose that $g$ satisfies the Polyak-\L{}ojasiewicz (P\L{}) condition (we consider other cases too). If the perturbations induced by $h$ are not too strong relative to $g$, 
we show that the SGD trajectory follows the gradient flow of $g$ and converges linearly to a neighborhood of the global solution.
Note that proving such a statement would be impossible when just assuming smoothness of $f$, as the function can have many local minima.

\noindent\textbf{Contributions.} Our contributions can be summarized as:
\begin{itemize}[leftmargin=12pt,nosep]
 \item We derive new and improved complexity estimates for \emph{perturbed SGD} methods---a class of randomized algorithms that perturb iterates by stochastic noise (similar to SGD) on a new class of structured non-convex functions.
 \item We derive worst-case complexity estimates of perturbed SGD on this function class. These estimates circumvent the lower complexity bounds that constrain the SGD analyses on general non-convex smooth functions~\cite{arjevani2019lower}. In particular,  %
 we characterize settings where perturbed SGD methods
 \begin{itemize}[nosep]
     \item converge linearly to the exact (or a neighborhood of) the global solution,
     \item or converge sub-linearly to the exact (or a neighborhood of) the global solution.
\end{itemize}
 Both these results improve over  traditional analyses which only show sublinear convergence to local minma or stationary points (which can be arbitrary far from the global minima).
\item Utilizing the insights developed in~\cite{kleinberg_alternative_2018},  we are able to link our convergence results to the behavior of SGD and demonstrate this connection via illustrative numerical experiments.
\end{itemize}

The code for all the experiments and plots in this paper has been uploaded to the following repository:\\
{\footnotesize
\href{https://github.com/mlolab/perturbed-sgd-demo}{https://github.com/mlolab/perturbed-sgd-demo}.}

\section{Related Works}
\textbf{Benefits of Injecting Noise:}\  It has been observed that 
the noise in the gradient can help SGD to escape saddle points~\cite{Ge2015:escaping} or achieve better 
generalization~\cite{Hardt16:trainfaster,Mou2018:sgld}.
This is often explained by arguing that SGD finds `flat' minima  with favorable generalization properties~\cite{Hochreiter1997:flat,Keskar2017,jastrzkebski2017three}, though also `sharp' minima can also generalize well~\cite{dinh2017sharp}.
These advantageous properties of SGD decrease as the batch size is increased~\cite{Keskar2017} or with variance reduction techniques~\cite{Defazio2019}.
Several authors proposed to artificially inject noise into the SGD process for improved generalization~\cite{neelakantan2015adding,chaudhari2017entropy,Plappert2017}, in particular in the context of large batch training~\cite{wen2018smoothout,Haruki2019,Lin2020}.

\begin{table*}[t]
    \caption{Comparison to related works on non-convex optimization. Oracle complexity for finding and $\epsilon$-approximate stationary point $\norm{\nabla f(\xx)} \leq \epsilon$,  assuming Lipschitz gradients for all methods and Lipschitz Hessians for methods converging to second-order stationary points.
    The structural assumptions enable global convergence in certain cases.} \label{table:non-convex}
    \begin{minipage}{\linewidth}
    \begin{center}
    \renewcommand{\arraystretch}{1.2}    
    \scalebox{0.7}{
    \begin{tabular}{c|c|c|c|c}
    \toprule
    \multicolumn{1}{c}{Output} & \multicolumn{1}{c}{Assumptions} & \multicolumn{1}{c}{Oracle} & \multicolumn{1}{c}{Method}  & Rate \\ \midrule
\multicolumn{1}{c|}{\multirow{2}{*}{\makecell{First-order\\ stationary point}}}&\multicolumn{1}{c|}{\multirow{2}{*}{\makecell{Gradient\\ Lipschitz}}}   & \multicolumn{1}{c|}{\multirow{2}{*}{Gradient}} & \multicolumn{1}{c}{\cite{ghadimi_accelerated_2016}} &$\cO(\epsilon^{-2})$ \\ %
                                 &                  &                             & \multicolumn{1}{c}{\cite{pmlr-v70-carmon17a}} & $\Tilde{\cO}(\epsilon^{-1.75})$ \\ \cline{1-5}
\multicolumn{1}{c|}{\multirow{13}{*}{\makecell{Second-order \\ stationary point\footnote{$\lambda_{\min}(\nabla^2 f(\xx)) \geq -\sqrt{\rho\epsilon}$, where $\rho$ denotes the Lipschitz constant of the Hessian.}}}} &\multicolumn{1}{c|}{\multirow{13}{*}{\makecell{Function,\\ Gradient\\ and Hessian\\ Lipschitz}}}  & \multicolumn{1}{c|}{\multirow{1}{*}{Hessian}} & \multicolumn{1}{c}{\cite{nesterov_cubic_2006}} &$\cO(\epsilon^{-1.5})$ \\ \cline{3-5}
&& \multicolumn{1}{c|}{\multirow{2}{*}{\makecell{Hessian-vector\\product}}} & \multicolumn{1}{c}{\cite{Carmon2016AcceleratedMF}} &$\Tilde{\cO}(\epsilon^{-2})$ \\ %
&&                                                                          & \multicolumn{1}{c}{\cite{agarwal_stoc}} &$\Tilde{\cO}(\epsilon^{-1.75})$  \\\cline{3-5}
&& \multicolumn{1}{c|}{\multirow{2}{*}{\makecell{Gradient}}}            & \multicolumn{1}{c}{\cite{jin_how_2017}} &$\Tilde{\cO}(\epsilon^{-2})$ \\%
&&                                                
& \multicolumn{1}{c}{\cite{pmlr-v75-jin18a}} &$\Tilde{\cO}(\epsilon^{-1.75})$\\ \cline{3-5}
&& \multicolumn{1}{c|}{\multirow{9}{*}{\makecell{Stochastic \\ Gradient}}} & \multicolumn{1}{c}{\cite{pmlr-v65-zhang17b}} &${\rm poly}(\epsilon^{-1})$ \\ %
& & & \multicolumn{1}{c}{\cite{Ge2015:escaping}} &$\Tilde{\cO}(\epsilon^{-4})$ \\ \cline{4-5}
&&                                                           & \cite{pmlr-v99-fang19a} &\multicolumn{1}{c}{\multirow{4}{*}{\makecell{$\Tilde{\cO}(\epsilon^{-3.5})$}}}\\ %
&&                                                           & \cite{NEURIPS2018_tripuraneni} & \\ %
&&                                                           & \cite{NEURIPS2018_allen_zhu_natasha2} & \\ %
&&                                                           & \cite{pmlr-v84-reddi18a} & \\ \cline{4-5}
&&                                                           & \multicolumn{1}{c}{\cite{NIPS2017_lei}} &\multicolumn{1}{c}{\multirow{1}{*}{\makecell{$\Tilde{\cO}(\epsilon^{-3.33})$}}}\\ %
&&                                                           & \multicolumn{1}{c}{\cite{NEURIPS2018_zhu_neon2}} &\multicolumn{1}{c}{\multirow{1}{*}{\makecell{$\Tilde{\cO}(\epsilon^{-3.25})$}}}\\ %
&&                                                           & \multicolumn{1}{c}{\cite{NEURIPS2018_cong}} &$\Tilde{\cO}(\epsilon^{-3})$ \\ \cline{1-2} \cline{4-5}
\textbf{Global Minima}&      \makecell{Assumption~\ref{ass:structure},\\  $\nabla g(\xx)$ Lipschitz\\ $g$ convex}  &           & \textbf{This paper} &$\tilde\cO(\log(\epsilon^{-1}) + \epsilon^{-1})$\\ 
\bottomrule
    \end{tabular}
    }%
    \end{center}
    \end{minipage}
    \end{table*}

\textbf{Approximate Minima in Non-Convex Functions:}\  Despite their NP-hardness, several works have studied non-convex optimization problems. Standard analysis for smooth functions can guarantee convergence to a first order stationary point ($\norm{\nabla f(\xx)} \leq \epsilon$) only~\cite{Ghadimi2013,ghadimi_accelerated_2016} at rate $\cO(\epsilon^{-2})$. Recently, there has been much interest in second-order stationary points, where $\epsilon$-SOSP is defined as  $\norm{\nabla f(\xx)}\leq \epsilon, \lambda_{\min}(\nabla^2 f(\xx))\geq -\sqrt{\epsilon}$~\cite{Ge2015:escaping,NEURIPS2018_zhu_neon2,xu2018neon}. If all saddle points are strict, then all $\epsilon$-SOSP are approximate local minima~\cite{jin_how_2017}. Thus, convergence to $\epsilon$-SOSP allows us to escape all saddle points. While SGD guarantees $\cO(\epsilon^{-4})$ convergence to $\epsilon$-SOSP, utilizing acceleration and second-order approximations improves it to $\cO(\epsilon^{-3.5}$)~\cite{agarwal_stoc,Carmon2016AcceleratedMF,pmlr-v70-carmon17a,pmlr-v75-jin18a,NEURIPS2018_jin}. Other methods, with same or slightly better rates, utilize efficient subroutines~\cite{NEURIPS2018_allen_zhu,NEURIPS2018_allen_zhu_natasha2}, negative curvature of the loss~\cite{NEURIPS2018_xu,NEURIPS2018_yu,NEURIPS2018_cong}, adaptive  regularization~\cite{xu_newton-type_2020,NEURIPS2018_tripuraneni,nesterov_cubic_2006} and variance reduction~\cite{zhou_stochastic_2019,pmlr-v84-reddi18a,NIPS2017_lei}.
Our work provides much stronger guarantees as we show convergence to a neighborhood of the global minima, in turn escaping both saddle points and local minima, with a much simpler algorithm using only first-order stochastic gradients. We provide a comprehensive comparison of these methods in Table ~\ref{table:non-convex}.

\textbf{Smoothing:}\ Injecting artificial noise is classically also known as \emph{smoothing} or \emph{convolution}~\cite{dalembert1754,dominguez2015} and has found countless applications in various domains
and communities.
In the context of optimization, 
smoothing has been used at least since the 1960s in
\cite{Rastrigin1963,Matyas1965,Schumer1968}.
While most proofs apply to the convex setting only~\cite{nemirovskij1983problem,nesterov_random_2017,stich2014convex}, smoothing is more prominently used in heuristic search procedures for non-convex problems~\cite{Blake1997,Hansen2001:cma}.
One of the outstanding features of the smoothing technique is that it allows to reduce the optimization complexity of non-smooth optimization problems~\cite{duchi_randomized_2012,nesterov_smooth_2005}.

\textbf{Compositional structure:}\  Often in machine learning settings, an inherent structure $f=g+h$ is explicitly known, for instance when one term denotes a regularizer. 
In this case, optimization methods can be designed that exploit favorable properties of the regularizer (such as strong convexity)~\cite{Duchi2010comid,Nesterov2013}.
However, this is different from our approach, as these algorithms need to have explicit knowledge of the regulariser. We, instead, use the structure~\eqref{eq:structure} only as an analysis tool \citep[opposed to e.g.][]{chen_stochastic_2017}, while the algorithm has only access to stochastic gradients of $f$.

\textbf{Approximately convex functions:}
Another approach for analysis of non-convex functions investigates weaker forms of convexity.
The most common formulations include P\L{} functions~\cite{Polak1963:pl,Lojasiewicz1963:pl,karimi_linear_2016}, where all minima are global minima, star-convex functions~\cite{zhou_sgd_2019,lee_optimizing_2016}, which are convex about the minima and approximately convex functions, which differ from convex functions by a bounded constant~\cite{pmlr-v65-zhang17b,NEURIPS2018_jin,pmlr-v40-Belloni15}. These functions are analyzed using standard techniques used for convex function, as they slightly relax the notion of convexity. \citet{necoara_linear_2016} provide a survey of when this analysis can lead to linear convergence. The class of non-convex functions that we consider subsume most mild cases of non-convexity like P\L, star-convexity or approximate convexity, by setting $h(x)$ to be bounded. Further, our framework can also be extended to stronger ones like quasar-convexity~\cite{hinder_near-optimal_2019, jin_convergence_2020}, by appropriately setting the value of $g$.

\textbf{Non-convex smoothing:}\ A theoretical connection between stochastic optimization and smoothing as been established in~\cite{kleinberg_alternative_2018}.
They study smoothing with distributions with bounded support (while we do not make this restriction) and prove convergence under the assumption the smooth $f_\cU$ is star convex~\cite{hinder_near-optimal_2019}.
In~\cite{hazan_graduated_2015} a graduated smoothing technique was analyzed under the assumption the smoothed function is strongly convex on a sufficiently large neighborhood of the optimal solution. Further, smoothing has been used in the context of derivative free optimization or in Langevin dynamics in non-convex regimes, most notably in~\cite{NEURIPS2018_jin,pmlr-v65-zhang17b,pmlr-v40-Belloni15}, however these works do not show global linear convergence in stronger paradigms of non-convexity.

\section{Notation}
For the reader's convenience, we summarize here a few standard definitions~\cite{nesterov_cvx_opt}. 
We say that a function $f \colon \R^d \to \R$ is $L$-smooth if its gradient is $L$\nobreakdash-Lipschitz continuous:
\begin{align}
\norm{\nabla f(\xx)- \nabla f(\yy)} \leq L \norm{\xx-\yy}\,, & & \forall \xx,\yy \in \R^d.  \label{def:Lsmooth}
\end{align}
A function $f \colon \R^d \! \to  \! \R$ is $\mu$-strongly convex for $\mu\geq 0$, if \looseness=-1
\begin{align*}
\lin{\nabla f(\xx)-\nabla f(\yy),\xx-\yy} \geq \mu \norm{\xx-\yy}^2\,, \, \forall \xx,\yy \in \R^d\,
\end{align*}
Sometimes relaxations of this condition are considered. A function $f \colon \R^d \to \R$ satisfies the Polyak-\L{}ojasiewicz ($\mu$-P\L) condition with respect to $\xx^\star$ if
\begin{align}
2\mu(f(\xx) - f^\star)\leq  \norm{\nabla f(\xx)}^2, \qquad \forall \xx \in \R^d\,. \label{def:pl}
\end{align}
Here,   $f^\star =  \min_{\xx \in \R^d}f(\xx) $. P\L{} functions can have multiple global minima, but for strongly convex functions, $\xx^\star = \argmin_{\xx \in \R^d}f(\xx)$ is unique.
We provide additional useful standard consequences of these inequalities in Appendix~\ref{sec:technicalstuff}.

\section{Perturbed SGD}
\label{sec:new-algorithm}
Our main goal is to study the convergence of SGD on problem~\eqref{eq:problem}. The SGD algorithm
 is defined as
\begin{align}
\begin{split}
    \xx_{t+1} &:= \xx_t - \gamma \nabla f(\xx_t,\bxi_t)  \,,
\end{split}\tag{SGD}
\end{align}
for a constant stepsize $\gamma$ and a uniform stochastic sample $\bxi_t \sim \cD$. This update can equivalently be written as \vspace{-2mm}
\begin{align}
\begin{split}
    \xx_{t+1} &= \xx_t - \gamma \nabla f(\xx_t) + \gamma \ww_t \,,
\end{split}\tag{SGD}\label{eq:SGD}
\end{align}
by defining $\ww_t := \nabla f(\xx_t) - \nabla f(\xx_t, \bxi_t)$. 
Let $\ww_t \sim \cW(\xx_t)$, where $\cW(\xx_t)$ denotes the distribution of $\ww_t$, which can depend on the iterate $\xx_t$.

\textbf{Standard approach.}
Standard analyses of SGD on non-convex $L$-smooth functions typically derive an upper bound on the \emph{expected} one step progress~\citep[e.g. Thm. 4.8 in][]{Bottou2018:book}. This gives
\begin{align*}
    \E f(\xx_{t+1}) \leq f(\xx_t) - \gamma  \norm{\nabla f(\xx_t)}^2 + \frac{\gamma^2 L}{2} \underbrace{{\rm \mathbb{V}ar}(\ww_t)}_{\geq 0}\,.
\end{align*}
However, following this methodology, stochastic updates can only guarantee a \emph{smaller} expected one step progress than the gradient method, as the variance is always positive.

\textbf{Our approach.}
To circumvent the aforementioned limitation, we adopt two key changes. First, by utilizing the structure~\eqref{eq:structure} we study the one step progress on $g$ and secondly, we formulate the algorithm slightly differently. Concretely, we study \textbf{perturbed SGD} (Algorithm~\ref{alg:smooth_sgd}) that we formally define as
\begin{align}
\begin{split}
    \xx_{t+1} &= \xx_t - \gamma \nabla f(\xx_t-\uu_t,\bxi_t) \,,
\end{split}\tag{perturbed SGD}\label{eq:perturbed_sgd}
\end{align}
for a random perturbation $\uu_t \sim \cU(\xx_t)$.
For this method, the expected one step progress can be estimated as,
\begin{align}
    \E g(\xx_{t+1}) \leq &g(\xx_t) - \gamma  \underbrace{\lin{\nabla g(\xx_t)
    ,\E_{\uu_t,\bxi_t}[\nabla f(\xx_t - \uu_t, \bxi_t)]}}_{\text{\circledone}} \notag \\
    &\quad+ \underbrace{\frac{\gamma^2 L}{2} {\rm \mathbb{V}ar}_{\uu_t,\bxi_t}(\nabla f(\xx_t - \uu_t, \bxi_t))}_{\text{\circledtwo}}\,.\label{eq:psgd_one_step}
\end{align}

The above formulation allows us to obtain  larger progress than standard analysis, by the virtue of considering $g$ and by using an appropriate smoothing distribution $\cU$. To establish convergence, we will impose appropriate conditions on terms \circledone~ and ~\circledtwo~ in~\eqref{eq:psgd_one_step}, which forms the basis for our Assumptions in Section~\ref{sec:assumptions}.

It is easy to see that perturbed SGD comprises SGD, for instance when $\uu_t \equiv 0$ a.s. 
However, there are more possibilities to trade-off the randomness in $\bxi_t$ and $\uu_t$. For instance, assume for illustration that perturbed SGD can access noiseless samples of the gradient, i.e.\,$\nabla f(\xx_t-\uu_t)$, and that $f$ is quadratic function $f$ with full rank Hessian $A$. Then it is still possible to simulate SGD by defining $\uu_t = \gamma A^{-1} \ww_t$ as can be seen from
\begin{align*}
    \nabla f(\xx_t-\uu_t)= A \xx_t - A\uu_t \equiv A\xx_t + \gamma \ww_t \,.
\end{align*}
In Section~\ref{sec:connection}, we derive  more general connections between perturbed SGD and vanilla SGD.

To summarize, we introduce perturbed SGD with the purpose to study the 
impact of smoothing $\uu \sim \cU$ and stochastic gradient noise $\bxi \sim \cD$ separately. 
Perturbed SGD is illustrated in Algorithm~\ref{alg:smooth_sgd} and implements a stochastic smoothing oracle by only accessing stochastic gradients of $f$.
For simplicity, we assume constant step length $\gamma$.
\begin{algorithm}[bht]
\caption{Perturbed SGD}
\begin{algorithmic}[1]\label{alg:smooth_sgd}
\REQUIRE $\gamma, f(\xx), T, \cU(\xx), \xx_0$
\FOR{$t=0$ to $T-1$}
    \STATE sample $\uu_t \sim \cU(\xx_t)$  \hfill $\triangleright$  smoothing distribution
    \STATE sample $\bxi_t \sim \cD$   \hfill $\triangleright$ (mini-batch) data sample
    \STATE $\xx_{t+1} = \xx_{t} - \gamma \nabla f(\xx_t - \uu_t,\bxi_t)$ \hfill $\triangleright$ SGD update 
\ENDFOR \hfill (or ADAM/momentum)
\end{algorithmic}
\end{algorithm}

\section{Setting and Assumptions}
We will now introduce the main assumption on the objective function $f$ with structure~\eqref{eq:structure} and give an illustrative example.

\subsection{Smoothing}
To formalize the notion of perturbations (i.e.\ the $\uu_t$'s in Algorithm~\ref{alg:smooth_sgd}), we utilize the framework of smoothing~\cite{duchi_randomized_2012}. 
Convolution-based smoothing of a function $f \colon \R^d \to \R$ is defined as\footnote{If $\cU$ is symmetric, this is equivalent to the more standard definition $E_{\cU}[f(\xx+\uu)]$.}
\begin{align}
 f_{\cU}(\xx) := E_{\uu \sim \cU} f(\xx - \uu)\,, \qquad \forall \xx \in \R^d\,,
 \label{def:smoothing}
\end{align}
for  a probability distribution $\cU$ (sometimes we will allow $\cU(\xx)$ to depend on $\xx$).

Smoothing is a linear operator $(g+h)_\cU = g_\cU + h_\cU$ and when $f$ is convex, then $f_\cU$ is convex as well. 
The smoothing~\eqref{def:smoothing} cannot be computed exactly without having access to $f$, but one can resort to a stochastic approximation in practice. For a given $f$, we can query stochastic gradients of $\nabla f_\cU$ by sampling $\uu \sim \cU$ and evaluating $\nabla f(\xx-\uu)$.
Many works that analyze smoothing need to formulate concrete assumptions on the smoothing distribution $\cU$, for instance that variance $\E_{\uu \sim \cU(\xx)}\norm{\uu}^2 \leq \zeta^2$ is bounded by a parameter $\zeta^2 > 0$. This is, for instance, satisfied  for smoothing distributions with bounded support \citep[see][]{duchi_randomized_2012} or subgaussian noise, in particular for the normalized Gaussian kernel $\uu \sim \cN(\0,\zeta^2/d\,\mI_d)$. In our case, we do not need to formulate such an assumption on $\cU$ directly, instead we formulate a new assumption that jointly governs both smoothing and stochastic noise in the next section. 

\subsection{Main Assumptions}
\label{sec:assumptions}
 As mentioned earlier, these assumptions seek to improve the one step progress for perturbed SGD (Algorithm~\ref{alg:smooth_sgd}) by exploiting the key terms of $\uu$, ~\circledone~ and \circledtwo~ in~\eqref{eq:psgd_one_step}---in Assumptions~\ref{ass:sigma} and~\ref{ass:structure} respectively.

We now list the main assumptions for the paper.
\begin{assumption}[Stochastic noise]\label{ass:sigma}
The stochastic noise is unbiased, $\E_{\bxi \sim \cD} f(\xx,\bxi)=f(\xx)$, the smoothing distribution is zero-mean and  $\E_{\uu \sim \cU(\xx)}[\uu] = 0$,  
and there exist parameters $\sigma'^2 \geq 0$, $M' \geq 0$, such that after smoothing with $\cU(\xx)$, $\,\forall \xx \in \R^d$:
\begin{align}
\begin{split}
&\E_{\uu,\bxi} \norm{\nabla f(\xx-\uu,\bxi)-\nabla f_{\cU(\xx)}(\xx)}^2 \\
&\qquad\qquad\qquad\qquad\quad\leq \sigma'^2 + M' \norm{ \nabla f_{\cU(\xx)} (\xx) }^2\,. \label{def:noise}
\end{split}
\end{align}
\end{assumption}
Note that $\E_{\uu,\bxi} \nabla f(\xx-\uu,\bxi)=\nabla f_{\cU(\xx)}(\xx)$. Therefore~\eqref{def:noise} allows us to bound the variance term \circledtwo{} in~\eqref{eq:psgd_one_step}. 
This extends  the standard noise assumption in SGD settings~\cite{Bottou2018:book,stich_unified_2019} which are of the form $\sigma^2 + M\norm{\nabla f(\xx)}^2$ (we recover this assumption when $\uu \equiv 0$, a.s.). While in non-convex settings this prior assumption is could be restrictive (as $\norm{\nabla f(\xx)}^2$ is small for stationary points, enforcing large $\sigma'$),  in contrast, $\|\nabla f_{\cU(\xx)}(\xx))\|^2$ will still be  large at saddles or sharp local minima, and thus in general $\sigma'$ in~\eqref{def:noise} can be chosen much smaller.

\begin{remark}\label{rem:var_split}
If the smoothing distribution, $\cU(\xx)$ has variance bounded by $\zeta^2 + Z \norm{\nabla f_{\cU(\xx)}(\xx)}^2$, and the variance of stochastic gradients have variance bounded as $\sigma^2 + M\norm{\nabla f_{\cU(\xx)}(\xx)}^2$, for some $\sigma^2,\zeta^2,M,Z \geq 0$, then under independence of $\cU$ and $\cD$ and L-smoothness of $f$, we can choose the terms in Assumption~\ref{ass:sigma} as  
$\sigma'^2 := \sigma^2 + 2(L\zeta)^2$ and $M' := M+ 2(LZ)^2$. 
\end{remark}
The above remark allows us to separate the contributions of smoothing noise and stochastic noise. Further, setting the terms of smoothing ($\zeta,Z$) to $0$, we recover the standard assumptions for SGD with unbounded variance. A proof  of this remark is provided in Appendix~\ref{sec:technicalstuff}. 

We now shift our attention to the term \circledone{} in~\eqref{eq:psgd_one_step}.  Through the next assumption, we neatly tie this to the structure of the objective function in~\eqref{eq:structure}.
\begin{assumption}[Structural properties of $g$ and $h$]\label{ass:structure}
    The objective function $f\colon \R^d \to \R$ can be written in the form~\eqref{eq:structure}, with $g$ being $L_g$-smooth, and there exist parameters $0 \leq m < 1$ and $\Delta \geq 0$, such that, $\forall \xx \in \R^d$:
    \begin{align}
    \begin{split}
     &\norm{\nabla f_{\cU(\xx)}(\xx)- \nabla g(\xx)}^2
     \leq \Delta  + m \norm{\nabla g (\xx)}^2 \,. \label{def:omega}   
    \end{split}
    \end{align}
    \end{assumption}

While this function does not explicitly clarify the role of $h$, to illustrate we can split the term on LHS as $\nabla h_{\cU(\xx)}(\xx) + (\nabla g_{\cU(\xx)}(\xx) - \nabla g(\xx))$. The difference term $(g_{\cU(\xx)}(\xx) - \nabla g(\xx))$ can be bounded if $\cU(\xx)$ has bounded variance and $g$ is smooth. The purpose of this assumption then becomes controlling $\nabla h_{\cU(\xx)}(\xx)$, which essentially is the non-convex perturbation in $f$. Note that this  allows possibly unbounded $h$, however after smoothing, $\nabla h_{\cU(\xx)}(\xx)$ must be dominated by $\nabla g(\xx)$. This assumption is an extension of biased gradient oracles of~\citet{ajalloeian_analysis_2020}. 

Assumption~\ref{ass:structure} covers a large family of non-convex functions, including P\L{} and convex functions trivially. The ability of $\cU$ in reducing the  non-convexity of $h$ is quantified by $m$ and $\Delta$. Setting $m = 0$, we are able to handle bounded non-convex functions $h$. 

The above assumption also allows us flexibility in choosing $\cU$. For most problems, a family of distributions satisfy this assumption, with $m$ and $\Delta$ dependent on  which distribution we pick from this family. Therefore, the distribution $\cU$ is not completely problem dependent. We describe the effects of this Assumption and the freedom in choosing $\cU$ using an illustrative example.

\subsection{Illustrative Example}
We provide an illustrative example which satisfies our assumptions while displaying a high degree of non-convexity.
\label{sec:example}
Consider the following 1-dimensional function,
\vspace{-2mm}
\begin{align}\label{ex:1d}
    f(x)  = x^2 + ax\sin(bx)    \,,
\end{align}
for parameters $a,b >0$. We can choose $g(x) = x^2$ as the convex part, while $h(x) = ax\sin(bx)$ denotes the possibly unbounded non-convex perturbation.
 For $ab \geq 2$, this function can have infinitely many local minima, arbitrarily far away from its global minima. 

Even after smoothing with a Gaussian distribution $\cN(0,\zeta^2)$, the non-convex perturbations do not disappear, and it cannot be convex for any $\zeta$ (for more details see Appendix~\ref{sec:toyexampleappendix}).
However, these perturbations become smaller with respect to $g$ for larger $\zeta$, as shown in Fig.~\ref{fig:example}. 
This (provably)  allows the function to satisfy Assumption~\ref{ass:structure} for  $m$ and $\Delta$, which are dependent on $\zeta$, thus allowing us flexibility in the choice of distribution $\cU$.

\subsection{More Examples}
Our settings also cover `valley functions', described by ~\citet{hazan_graduated_2015}, eg for $x = (x_1, x_2,\ldots,x_d)^\intercal \in \R^d, \alpha, \lambda >0$, 
\begin{align*}
        f(\xx) = 0.5 \norm{\xx}^2 - \alpha e^{-\frac{x_1 - 1}{2\lambda^2}}
\end{align*}
These are non-convex functions with sharp local minima (in this case at $x = (1,0,0,\ldots, 0)^\intercal$, with $\lambda$ deciding the sharpness)  and resemble the loss surfaces of simple NNs.
We can also handle problems with bounded non-convexity which are common in practical learning settings. For instance, consider the training of a classifier in the presence of random label noise. A common solution approach for these problems is to modify the surrogate loss function to attain unbiased estimators---however this new optimization target might not be convex, even when starting from a convex loss function (such as least square regression).
\citet[Theorem 6]{Natarajan2013:noisy} prove that this non-convex optimization target $f$ is uniformly close to a convex function $g$, i.e. $h$ is bounded.
The function classes we consider contains this class of problems, yet we also cover more general cases where $h$ is not uniformly bounded. We cover additonal examples in detail in Appendix~\ref{sec:add_examples}.

\begin{figure}[t]
\centering
\vspace{-1mm}
     \subfigure[$f_{\cU}(x)$ and $g(x) = x^2$ \label{fig:diff_smoothing}]{
         \includegraphics[width=0.4\textwidth,trim=0 10 0 20, clip]{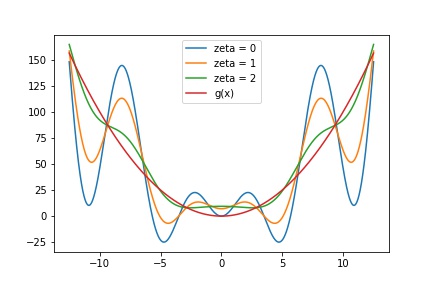}}
        \vspace*{-3mm}        %
     \subfigure[$\nabla f_{\cU}(x)$ and $\nabla g_{\cU}(x)$\label{fig:diff_smoothing_grad}]{ 
         \includegraphics[width=0.4\textwidth,trim=0 10 0 20, clip]{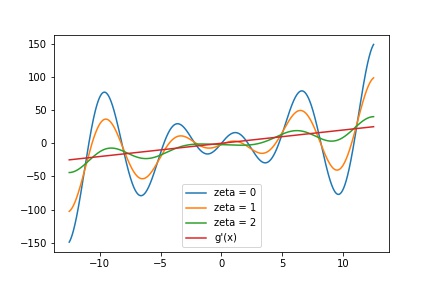}}
         \caption{Illustration of the effect of smoothing $f(x) = x^2 + 10x\sin(x)$ (blue) with the Gaussian kernel $\cN(0,\zeta^2)$ for different $\zeta \in \{0,1,2\}$.  $f_{\cU}(\xx) $ does not become convex even when choosing arbitrarily large $\zeta^2>0$.\label{fig:example}}
\end{figure}
\section{Convergence Analysis}\label{sec:main_results}
We now present the convergence analysis.
All the proofs, more detailed theorem statements, and additional extensions are deferred to Appendix~\ref{sec:conv_proof}.

\subsection{Gradient Norm Convergence}
\label{sec:thm1}
\begin{theorem} \label{thm:grad_norm_wo} 
Let $f$ satisfy Assumptions~\ref{ass:sigma} and~\ref{ass:structure}, and assume $g$ to be $L$-smooth , then there exists a stepsize $\gamma$  such that for any $\epsilon > 0$,
\begin{align*}
    T = \cO\bigg(\frac{M'+1}{\epsilon(1-m) + \Delta} + \frac{\sigma'^2}{\epsilon^2(1-m)^2 + \Delta^2}\bigg)L_g \cG_0
\end{align*}
iterations are sufficient to obtain $\frac{1}{T}\sum_{t=1}^T\E[\norm{\nabla g(\xx_t)}^2] = \cO(\epsilon + \frac{\Delta}{1-m})$, where %
 $\cG_t = \E[g(\xx_t)] -\min_{\xx \in \R^d} g(\xx)$.
\end{theorem}
This theorem shows that Algorithm~\ref{alg:smooth_sgd} converges to a neighborhood of a stationary point of $g$.  The size of the neighborhood depends on $\Delta$. When all stationary points of $g$ are  global minima (this is for instance the case for convex, star-convex, quasar-convex or quasi-convex functions), and $\Delta=0$, this theorem  shows global convergence of Perturbed SGD. We can show convergence with faster rates under additional assumptions on $g$.

\subsection{Convergence under P\L{} Conditions}
\label{sec:thm2}

\begin{theorem}\label{thm:pl_wo}
 Let $f$ satisfy Assumptions~\ref{ass:sigma} and~\ref{ass:structure}, and assume $g$ to be $\mu_g$-P\L{}. Then there exists a stepsize $\gamma$ such that for any $\epsilon > 0$,
 \begin{align*}
     T = \Tilde{\cO}\bigg((M' + 1)\log\frac{1}{\epsilon} + \frac{\sigma'^2}{\epsilon(1-m)\mu_g + \Delta}\bigg)\frac{\kappa}{1-m}
 \end{align*}
iterations are sufficient to obtain $\cG_T = \cO(\epsilon + \frac{\Delta}{\mu_g(1-m)})$, where $\kappa := \frac{L_g}{\mu_g}$ and $\Tilde{\cO}$ hides only log terms. 
\end{theorem}
If $\sigma'^2=0$ then this theorem shows linear convergence in $\cO\bigl( \frac{\kappa}{1-m} \log \frac{1}{\epsilon} \bigr)$ steps to a neighborhood of the global solution (and to the global solution when $\Delta=0$). When $\sigma'^2$ is large, the rate is dominated by the second term, $\cO\bigl( \smash{ \frac{\sigma'^2}{\epsilon (1-m)^2} } \bigr)$. This matches the $\cO\bigl( \smash{\frac{\sigma^2}{\epsilon}}\bigr)$ convergence rate of vanilla SGD on P\L{} functions. However, note that in our case $f$  does not need to be P\L{} to enjoy these convergence guarantees.

\subsection{Convergence under Strong Convexity}
\label{sec:strongconvex}
We now extend our results to the case when $g$ is strongly convex.
Note that while Theorem~\ref{thm:pl_wo} still applies (all strongly convex functions are P\L{}), applying this result  for P\L{} case admits a weaker convergence rate by a factor proportional to $\kappa$ in contrast to the improved result in Theorem~\ref{thm:str_cvx_wo}.
This result is not covered in prior frameworks, as matching convergence rates were previously only derived for $m < 1/\kappa$  \cite[Remark~7]{ajalloeian_analysis_2020}.
To achieve this, we slightly refine our Assumption~\ref{ass:structure}, ensuring we still are able to retain its expressivity.

\begin{assumption}[Structural properties]\label{ass:structure_m_str_cvx_wo}
The objective function $f\colon \R^d \to \R$ can be written in the form~\eqref{eq:structure} with $g$ being $L_g$-smooth,  and there exist parameters $\Delta \geq 0$,  $0 \leq m < 1$ such that, $\forall \xx \in \R^d$: 
\begin{align*}
 &\bigl| \bigl(\rr(\xx)\bigr)_g \bigr|^2 \leq  m \norm{\nabla g(\xx)}^2\,, & 
 &\bigl| \bigl(\rr(\xx)\bigr)_{g_{\perp}} \bigr| ^2 \leq \Delta \,, %
\end{align*} 
where $\rr(\xx) = (\nabla f_{\cU(\xx)}(\xx)  -\nabla g(\xx)) = (\rr(\xx))_g + (\rr(\xx))_{g_{\perp}}$. $(\rr(\xx))_g$ and $(\rr(\xx))_{g_{\perp}}$ denote the components of $\rr(\xx)$, along the direction of $\nabla g(\xx)$ and perpendicular to it, respectively.
\end{assumption}
Our main idea is to split the bound in Assumption~\ref{ass:structure} to its respective components. Note that we can easily verify that this is stronger than Assumption~\ref{ass:structure} by computing  $\norm{\rr(\xx)}^2$. 

To ensure the same level of expressivity for both the structural assumptions, we can verify that they have similar worst-case scenarios for a biased oracle, that is, when $\rr(\xx)$ points in the opposite direction of $\nabla g$ with squared norm $m\norm{\nabla g(\xx)}^2$, ignoring the constant terms of $\Delta$. Thus, our new assumption can still deal with worst-case oracles obeying Assumption~\ref{ass:structure}  while still admitting a better analysis.

\begin{theorem}\label{thm:str_cvx_wo}
Let $f$ satisfy Assumptions~\ref{ass:sigma} and~\ref{ass:structure_m_str_cvx_wo}, and assume $g$ to be $L_g$-smooth and $\mu_g$-strongly-convex, then there exist non-negative weights $\{w_t\}_{t=0}^T$, with $W_T = \sum_{t=0}^T w_t$ and stepsize $\gamma$ such that for any $\epsilon > 0$,
 there exist , such that 
\begin{align*}
    T = \Tilde{\cO}\bigg(\kappa (M' + 1 )\frac{m_{+}}{m_{-}}\log \frac{1}{\epsilon} + \frac{2(\sigma'^2+\Delta(M' + 1))}{\mu_g\epsilon m_{-} + 4\Delta}\bigg)
\end{align*}
iterations are sufficient to obtain $\frac{1}{W_T}\sum_{t=0}^T w_t \cG_t = \cO(\epsilon + \frac{4\Delta}{\mu_g m_{-}})$, where  $m_{-} = (1-\sqrt{m})^2$ and  $m_{+}=(1+\sqrt{m})^2$.
\end{theorem}
Comparing Theorems~\ref{thm:pl_wo} and~\ref{thm:str_cvx_wo}, we find that the $\kappa$ dependence is no longer present in the noise term, while our proof holds for arbitrary $m < 1$. Thus, we have addressed both the problems which we mentioned at the start of this subsection. However, this does not come for free, as the convergence rate is inversely proportional to $(1 - \sqrt{m})$, instead of $1-m$, in the P\L{} case and $1 - \sqrt{m} < 1 - m$. Also, we have a larger noise term ($\sigma'^2 + \Delta(M' + 1)$), than with P\L{}, which also depends on $\Delta$.

\subsection{Discussion of Results}
Our convergence results show convergence to the neighborhood of minima of $g, \, \xx_g^\star = \argmin_{\xx\in\R^d}g(x)$. While this does not directly imply convergence in terms of $f$, we can apply assumptions on $h$ so that it does. If $h$ is bounded, our convergence results hold for $f$ within a neighborhood defined by the bound on $h$.

For convergence in iterates we can characterize the $\norm{\xx_g^\star - \xx^\star}$ in terms of the non-convexity $h$. The following lemma provides this bound for strongly-convex $g$.
\begin{lemma}\label{lem:f_g_minima}
If $g$ is $\mu_g$-strongly-convex, 
\begin{align*}
    \norm{\xx^\star - \xx_g^\star}^2 \leq \frac{2}{\mu_g}(h(\xx^\star) - h(\xx_g)) \,.
\end{align*}
\end{lemma}
Thus, it suffices that the 
difference of perturbations at the global minima of $f$ and $g$, i.e.\
$(h(\xx^\star) - h(\xx_g^\star))$, is bounded, in order to show convergence to a close neighborhood of $\xx^\star$. Note that this is much weaker than assuming bounded  $h$.
This ensures that our Perturbed SGD converges to a neighborhood of global minima of the non-convex function $f$ in presence of local minima.

Further, our convergence results rely on the size of the neighborhood $\Delta$. This neighborhood would depend on the choice of $\cU$. For our toy example~\eqref{ex:1d}, $\zeta^2$ decides the size of this neighborhood and this is under our control. Additionally, convergence to a neighborhood of global minima allows us to escape all local minima and saddle points which are far away and have poor function value. We illustrate this further through experiments in Section~\ref{sec:experiments}.

\subsection{Insights}
\label{sec:insights}
We have derived convergence results under our novel structural assumption~\eqref{eq:structure} for Perturbed SGD (Alg.~\ref{alg:smooth_sgd}).
Our results depict the impact of the smoothing $\cU$ and the stochastic noise $\cD$, and when $\cU\equiv 0$ a.s.\ (no smoothing), we recover the known convergence results for SGD. 

All convergence results depend on the joint effect of smoothing and stochastic noise, $\sigma'^2 = \sigma^2 + L^2 \zeta^2$ (see Remark~\ref{rem:var_split}). This means, that any smoothing with $\zeta^2 \leq \frac{1}{L^2}\sigma^2$ does \emph{not worsen} the convergence estimates one would get by analyzing vanilla SGD alone. 
Moreover, smoothing allows convergence to the minima of $g$, and to avoid local minima of $f$ at a linear rate. Note that this is much faster and simpler than existing methods~\cite{pmlr-v65-zhang17b,NEURIPS2018_jin} which can only converge to approximate local minima. In particular, smoothing $f$  with the scaled gradient noise $\frac{1}{L}\cD$ we get for free a method that enjoys much more favorable convergence guarantees than SGD~\cite{Ge2015:escaping}. %
But is it even necessary to implement Pertubed SGD,  or does vanilla SGD suffice? We argue in the next section that this might indeed be the case.
\section{Connection to SGD}
\label{sec:connection}
We explain how the analysis from the previous section is connected to the standard SGD algorithm. (that does not implement the smoothing perturbation $\uu \sim \cU(\xx)$ explicitly).

\subsection{Stochastic Online Setting}
\label{sec:online_setting}
This follows directly from insights in~\cite{kleinberg_alternative_2018}.
Let $\xx_t$ be the SGD iterates as defined in~\eqref{eq:SGD}, with noise $\ww_t \sim \cW(\xx_t)$, where $\cW(\xx_t)$ is the gradient noise distribution. \citet{kleinberg_alternative_2018} propose to study the alternate sequence $\yy_t$ defined as
\begin{align*}
\yy_{t+1} = \xx_t - \gamma \nabla f(\xx_t) \,.
\end{align*}
Let $\zz_t$ define the iterates of Algorithm~\ref{alg:smooth_sgd} as defined in~\eqref{eq:perturbed_sgd}, with only smoothing, $\uu_t$,  and no gradient noise, $\xi_t$. Let $\uu_t \sim \cU(\xx_t)$, where $\cU(\xx_t)$ is the smoothing distribution.

\begin{figure}[t]
     \centering
\includegraphics[width=0.35\textwidth]{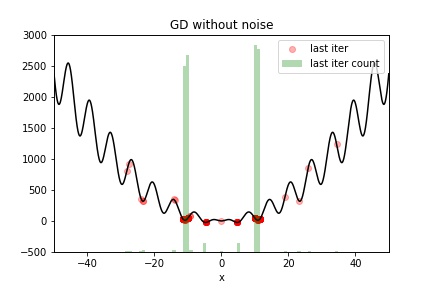}
\vfill 
\includegraphics[width=0.35\textwidth]{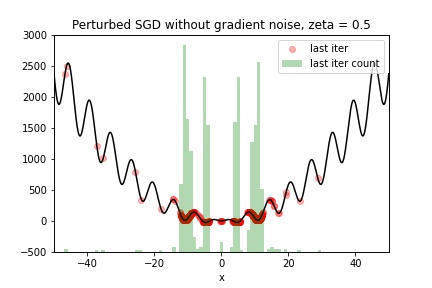}\vfill
\includegraphics[width=0.35\textwidth]{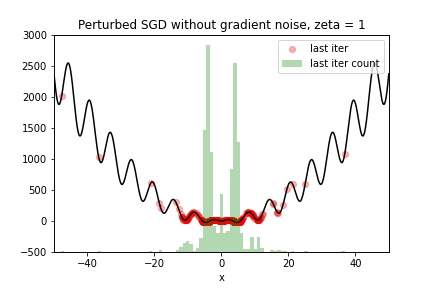}
    \label{fig:smooth_gd_1}
     \vspace{-2mm}
          \caption{Distribution of the last iterates from GD and Algorithm~\ref{alg:smooth_sgd} on the toy example~\eqref{ex:1d}. 
          We run Perturbed SGD and GD for $1000$ random initializations between $[-400,400]$ and for $T=100$ iterations. Here $\sigma =0, M=0$ and $\zeta \in \{0.5,1\}$ with Gaussian smoothing. We select the step size form a grid search over a grid of 4 step sizes from $[10^{-5},1]$, which are exponentially separated. For better visualization, we plot the locations and histogram of the last iterate from these runs, restricted to the interval $[-50,50]$.\label{fig:comp_gd}}
\end{figure}
\begin{lemma}[{Equality in Expectation, \cite[adopted from][]{kleinberg_alternative_2018}}]\label{lem:eq_in_exp}
    For $\xx_t, \yy_t$ and $\zz_t$ defined as above, if $\zz_0=\yy_0$ and $\cU(\xx_t)=\gamma \cW(\xx_t)$ for all $t \geq 0$, then
    \vspace{-2mm}
    \begin{align*} \E[\zz_{t}]=\E[\yy_t]  \,.
    \end{align*}
\end{lemma}
The proof for this lemma relies on induction. We show this for $t=1$, and refer the reader to~\cite{kleinberg_alternative_2018} for the proof.
Consider $\E[\yy_1]$, 
\begin{align*}
    \E[\yy_1] &= \E[\xx_0 - \gamma \nabla f(\xx_0)]\\
    &= \xx_0 - \E[\ww_0]  - \gamma \E[\nabla f(\yy_0 - \gamma\ww_0)]\\
    &= \xx_0  - \gamma \E_{\ww_0}[\nabla f(\yy_0 - \gamma\ww_0)]\\
    &= \zz_0 - \gamma\E_{\uu_0}[\nabla f(\zz_0 - \uu_0)]\\
&=  \E[\zz_1] \,.
\end{align*}
The first and second equation utilize the definition of $\yy_t$. In the third equation, we use the fact that $\ww_0$ is zero-mean, while in the fourth equation, we substitute $\uu_0 = \gamma \ww_0$, since $\cU(\xx_0) = \gamma\cW(\xx_0)$. 

This Lemma establishes the intuition, that SGD is performing approximately gradient descent on a smooth version of $f$. Note that we establish only a weak equivalence in expectation. However, the next lemma shows that even this weak equivalence is sufficient to use our main results from Theorem~\ref{thm:str_cvx_wo} for SGD analysis.
\begin{lemma}\label{lem:eq_expectation}
    Let $\xx_t, \yy_t$ and $\zz_t$ be as defined above. Define $\Bar{\yy}_T := \frac{1}{W_T}\sum_{t=0}^T\ww_t\E[\yy_t]$, for $\{\ww_t\}_{t=0}^T$ and $W_T$ as defined in Theorem~\ref{thm:str_cvx_wo}. If Lemma~\ref{lem:eq_in_exp} holds, $g$ is convex,  
    \begin{align*}
        g(\Bar{\yy}_T) - g(\xx_g^\star) \leq \frac{1}{W_T}\sum_{t=0}^T \ww_t\cG_t
    \end{align*}
    where $\xx_g^\star$ and $\cG_t$ are as defined before.
    \end{lemma}
    The above lemma is a straightforward application of Jensen's inequality and equality in expectation. The complete proof is presented in Appendix~\ref{sec:technicalstuff}.
We can now utilize the results of Thm.~\ref{thm:str_cvx_wo} for SGD iterates defined by $\yy_t$.
\subsection{Finite-Sum Setting}\label{sec:finite_sum}
We now explain the connection between SGD and Perturbed-GD for a finite-sum objective. Note that common machine learning applications follow a finite-sum structure, where the objective function is mean of training losses on all data samples of a dataset. This formulation allows us to empirically verify the connection between SGD and Perturbed-GD for common machine learning applications like Logistic Regression and neural networks.

Consider the finite-sum objective function,
  $f(\xx) = \frac{1}{n} \sum_{i=1}^n f_i(\xx)$, which is a sum of $n$ terms. 
    For SGD, at each step $t$,
    \begin{align*}
        \nabla f(\xx_t,\bxi) = \nabla f_i(\xx_t)
    \end{align*} 
    where $i$ is sampled uniformly at random from $[n]$. Thus, the noise in each gradient step, $\ww_t$, is,
    \begin{equation}
    \begin{aligned}
         \ww_t &= \nabla f_i(\xx_t) - \nabla f(\xx_t)\\ 
         &= \nabla f_i(\xx_t) - \frac{1}{n}\sum_{j=1}^n \nabla f_j(\xx_t) \,. \label{eq:finite_sum_full_batch}
     \end{aligned}
    \end{equation}

    To find an equivalent smoothing distribution, we can set $\cU(\xx) =\gamma\cW(\xx)$ as described above. However, the resulting distribution would require us to compute $\uu_t = \gamma\bigl(\nabla f_k(\xx_t) -  \frac{1}{n}\sum_{j=1}^n \nabla f_j(\xx_t)\bigr))$ for an uniformly at random sampled index $k$. This involves computation of a full batch gradient, rendering the resulting procedure very inefficient. To overcome this, we can define $\uu_t$ in the following way: \looseness=-1
    \begin{align}
        \uu_t = \gamma(\nabla f_k(\xx_t) - \nabla f_j(\xx_t)) \,, \label{eq:structuredsmoothing}
    \end{align}
    where $k,j$ are sampled uniformly at random from $[n]$. This results in an efficient oracle with variance
    \begin{align*}
        \E_{\cU(\xx_t)}[\norm{\uu_t}^2] = 2\gamma^2\E_{\cW(\xx_t)}[\norm{\ww_t}^2] \,.
    \end{align*}
    Note that this resembles the method implemented in
    \cite{Haruki2019} in a distributed setting.

\section{Numerical Illustrations}
\label{sec:experiments}
In this section we provide numerical illustrations to demonstrate that Perturbed SGD is able to escape local minima in contrast to gradient descent (GD) and to verify its connection to  SGD. \looseness=-1
\subsection{Escaping Local Minima}
We compare the performance of our Algorithm~\ref{alg:smooth_sgd} with GD on our toy example $f(x) = x^2 + 10x\sin(x)$ with $\cU = \cN(0,\zeta^2)$ smoothing. The results (averaged over 1000 independent runs) are illustrated in Figure~\ref{fig:comp_gd}.
For this function there are two global minima located near $\pm 4.7$.
We observe that while GD gets stuck at poor local minima most of the time, our algorithm is able to escape these local minima. Further, increasing smoothing by increasing $\zeta$ helps in escaping local minima, and allows  convergence to the minima of $g(x) = x^2$, which is close to the global minima of $f$.

\begin{figure}[tb!]
    \centering
\vfill
\subfigure[Uniform noise setting]
{\includegraphics[width=0.38\textwidth,trim=0 0 0 20,clip]{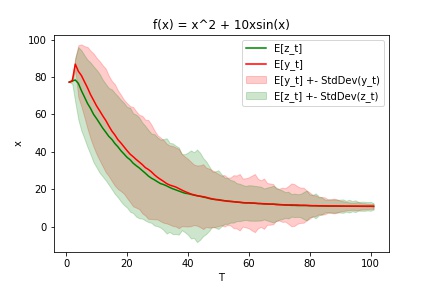}\label{fig:a}}
\vfill
\subfigure[Noise dependent on $\xx$]
{\includegraphics[width=0.38\textwidth,trim=0 0 0 20,clip]{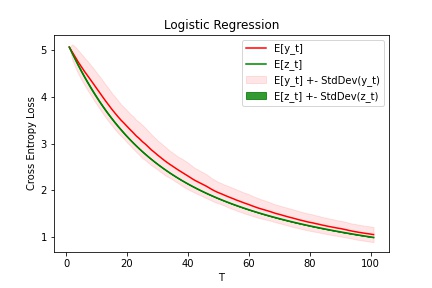}\label{fig:b}
}
\vfill\null
          \caption{Equivalent trajectories of SGD and Perturbed SGD. 
          Mean trajectories of $1000$ independent runs of SGD and Perturbed SGD with the same $\gamma$ selected by grid search, as described in Fig.~\ref{fig:comp_gd}, and $\zz_0 = \yy_0$. Solid lines depict mean and shaded areas standard deviations. 
          \label{fig:sgd_smoothsgd}
          }
\end{figure}

\subsection{Verifying Connections to SGD }
\label{sec:exp_conn_sgd}
We empirically demonstrate the connections between our algorithm and SGD in two settings, when noise is--   a) independent of $\xx$ (Section~\ref{sec:online_setting}) and b) dependent on $\xx$ (Section ~\ref{sec:finite_sum}).

For our first setting (depicted in Figure~\ref{fig:a}), we use our toy problem $f(x) = x^2 + 10x\sin(x)$. We fix the initial point for SGD as $x_0 = 100$ and $\zeta = 0.1$. We add a Gaussian noise sampled from $\cN(0,\sigma^2)$ to the gradients, where $\sigma^2 = \gamma \zeta^2$.

For our second setting (depicted in Figure~\ref{fig:b}), we consider a finite-sum objective. The stochastic noise arises from sampling one datapoint in the finite sum with replacement, and is thus dependent on $\xx$. We use logistic regression with cross entropy loss on the Digits dataset \cite{Dua:2019} from scikit-learn \cite{pedregosa2011scikit}. The dataset consists of $8\times 8$ images of handwritten digits from {\tt 0} to {\tt 9}, from which we use only images of {\tt 0} and {\tt 1}. For SGD, $\xx_0$ is sampled uniformly from $[-0.5,0.5]^{64}$. We choose the same sampling for $\cU(\xx)$, to obtain $\cU(\xx) = \gamma\cW(\xx)$.

For both of these cases, the mean trajectories for $\yy_t$ and $\zz_t$ are very close, verifying our analysis. For the uniform noise setting, the variances of the trajectories are also very similar. However, the variance for our algorithm is much smaller than SGD for the  logistic regression example.
Now, we  illustrate this connection for deep learning examples in 
the next section. We also analyze our toy example under high noise settings, which are described in Appendix~\ref{sec:add_exp}.
\subsection{Deep Learning Examples}
\label{sec:deep_learning_examples}

We further investigate the equivalence between SGD and Perturbed SGD for a standard deep learning problem---Resnet18~\cite{he2015deep}
on CIFAR10 dataset~\cite{Krizhevsky09learningmultiple}.
Note that in deep learning settings, our loss function is $f(\xx) = \frac{1}{n}\sum_{i=1}^n f_i(\xx)$, where $f_i(\xx)$ is the loss, in this case cross-entropy ,  for the $i^{th}$ datapoint in the dataset for network with weights given by $\xx$.

We compare Perturbed SGD with mini-batch SGD with batch size 128. 
In Section~\ref{sec:finite_sum}, we describe two possible implementations for the finite-sum setting---\eqref{eq:finite_sum_full_batch}
and~\eqref{eq:structuredsmoothing}. Since we require the full-batch gradient in each step of ~\eqref{eq:finite_sum_full_batch}, we cannot use this in deep learning settings with large dataset sizes.
In~\eqref{eq:structuredsmoothing}, we utilize only minibatch gradients, so we can apply it to deep learning problems. 
In our pytorch 
implementation, we break down Algorithm ~\ref{alg:smooth_sgd} into two steps--perturbation step which computes $\uu_t$, and the gradient step which updates parameters with $\nabla f(\xx_t + \uu_t,\bxi_t)$.

To verify the equivalence of SGD and Perturbed SGD, we need to ensure the same noise levels and the number of steps for both algorithms. 
We briefly describe how this is achieved for finite-sum implementation of Perturbed SGD described in~\eqref{eq:structuredsmoothing}.

For~\eqref{eq:structuredsmoothing}, the perturbation step and the gradient step have 3 times the noise of SGD, as the perturbation step has 2 times the noise of SGD. To ensure the same noise levels, we set the batch size for both steps as $128 \times 3 =  384$. To ensure the same number of steps as SGD in one epoch, we repeat perturbation + gradient step 3 times in each epoch.
\begin{figure}[t]
     \centering
     \subfigure[Training Accuracy\label{fig:tr_acc}]{
         \includegraphics[width=0.3\textwidth]{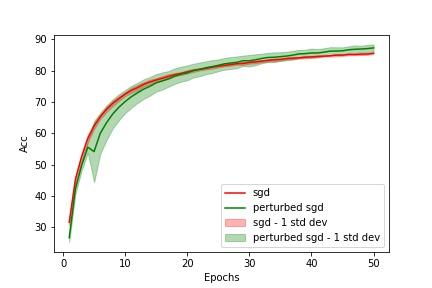}}
     \subfigure[Training Loss\label{fig:tr_loss}]{
         \includegraphics[width=0.3\textwidth]{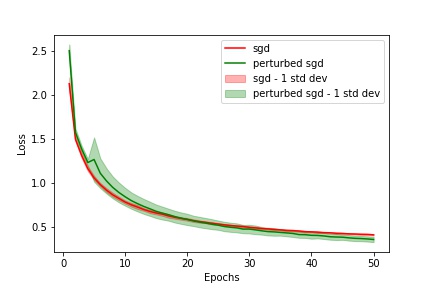}}
     \subfigure[Test Accuracy\label{fig:val_acc}]{
         \includegraphics[width=0.3\textwidth]{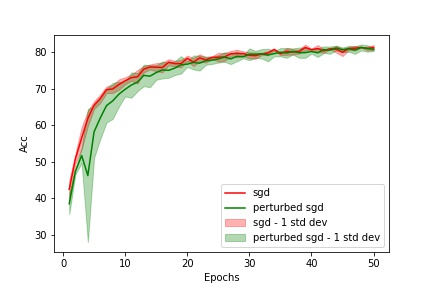}}
          \caption{Equivalent Trajectories for SGD and finite-sum implementations of Perturbed SGD (Algorithm~\ref{alg:smooth_sgd}) according to ~\eqref{eq:structuredsmoothing}.
          Mean trajectories after $5$ independent runs of SGD and Perturbed SGD with same $\gamma = 0.1$, momentum = $0.9$, weight decay = $10^{-4}$ for 50 epochs with same initialization and noise levels. Solid lines depict means and shaded areas standard deviations.\label{fig:dl_exp}
          }
\end{figure}

From Fig~\ref{fig:dl_exp}, we can see that the efficient finite-sum implementation of Perturbed SGD and SGD have very similar trajectories for training accuracy, training loss and validation accuracy. This verifies our claim of equivalence of SGD and Perturbed SGD on DL examples, with the same noise levels. Moreover, the variance is higher for  Perturbed SGD than SGD, despite similar gradient noise level, providing further motivation to investigate benefits of Perturbed SGD in generalization and escaping saddles~\cite{Ge2015:escaping}.

\section{Discussion and Outlook}
\label{sec:discussion}
There is a growing discrepancy between the theoretically weak complexity results for SGD and its  empirically strong performance, which is often observed  on  non-convex DL examples.
This is because the theoretical modeling of the functional class---typically smooth non-convex losses---does not reflect well the practical challenges.
To break this complexity barrier, we propose a new class of functions that allow us to justify why stochastic methods (SGD or Perturbed SGD) can provably avoid local minima and converge (at a linear rate) to a global optimal solution.
However, it remains an interesting open question to prove that our structural assumptions hold for real DL tasks.

We believe that it possible to develop more advanced versions of Perturbed SGD, such as counterparts of momentum SGD, ADAM, or variance reduced methods that are specifically designed for (hidden) composite functions. Another direction could aim at proving convergence results for SGD on targets with hidden structure in a more direct way, without the detour via Perturbed SGD.
Research in this direction may for example shed new light on why variance reduced methods struggle on non-convex tasks~\cite{Defazio2019} and can lead to more efficient training methods for neural networks in general. 
An analysis of Perturbed SGD that studies its generalization properties is another promising direction~\cite{foret2021sharpnessaware}.

{\small
\bibliography{references}
\bibliographystyle{icml2022mod}
}

\newpage
\appendix
\onecolumn

\vspace{1cm}

\section{Additional Technical Tools}
\label{sec:technicalstuff}

We list here a few useful properties, sometimes used in the proofs. Further, we also provide missing proofs and additional analysis for  Remark~\ref{rem:var_split} and Lemma~\ref{lem:eq_in_exp} in Section~\ref{sec:connection}.

\subsection{On Smooth and Convex Functions}
\label{sec:smooth}
We first provide additional definitions and formulations for smooth functions, which we will use later.

A function is $\mu$-star-convex with respect to $\xx^\star$ if
 \begin{align}
 \lin{\nabla f(\xx)-\nabla f(\xx^\star),\xx-\xx^\star} \geq \mu \norm{\xx-\xx^\star}^2 \,, \qquad \forall \xx \in \R^d\,.  \label{def:starconvex}
 \end{align}
Strongly convex functions are both P\L{} and star convex.

The smoothness assumption~\eqref{def:Lsmooth} is often equivalently written as
\begin{equation}
    \abs{ f(\yy) - f(\xx) + \langle\nabla f(\xx), \yy - \xx\rangle } \leq  \frac{L}{2}\norm{\yy - \xx}^2, \qquad \forall \xx,\yy \in \R^d \label{eq:quadratic}
\end{equation}
\begin{remark}
Note that if a function $f$ is $L-smooth$ and has a minimizer $\xx^\star \in \argmin_{\xx \in \R^d}f(\xx)$, then it satisfies
\begin{equation}
    \norm{\nabla f(\xx)}^2 \leq 2L(f(\xx) - r(\xx^\star)) \qquad \forall \xx \in \R^d\,.
\end{equation}
\end{remark}
\begin{proof}
Let $\yy = \xx - \frac{1}{L}\nabla f(\xx)$, then, substituting these $\xx$ and $\yy$ in  above definition --
\begin{align*}
    \norm{\nabla r(\xx)}^2 \leq 2L(r(\xx) - r(\xx - \frac{1}{L}\nabla f(\xx))) \,.
\end{align*}
Since $r(\xx - \frac{1}{L}\nabla f(\xx)) \geq r(\xx^\star)$, we can substitute this in the upper bound.
\end{proof}

Strong convexity is often written as
\begin{align}
    f(\yy) \geq f(\xx) + \lin{\nabla f(\xx),\yy-\xx} + \frac{\mu}{2}\norm{\yy-\xx}^2 \qquad \forall \xx,\yy \in \R^d\,.
\end{align}

\subsection{Proof of Remark~\ref{ass:sigma}}
\label{sec:ass_sigma_analysis}
To prove Remark~\ref{rem:var_split}, we first restate a more general version of the assumptions  on the smoothing distribution $\cU(\xx_t)$ and noise distribution $\cD$ (in the main text we assumed $Z=0$ for simplicity). 
\begin{assumption}[Smoothing noise]\label{ass:zeta}
For given $f\colon \R^n \to \R$, the smoothing distribution $\cU(\xx)$ is zero-mean $(\E_{\uu \sim \cU(\xx)}\uu=\0)$, can possibly depend on $\xx \in \R^d$ and there exists constants ($\zeta^2\geq 0,Z^2 \geq 0$) such that the variance can be bounded as 
\begin{align}
 \E_{\uu \sim \cU(\xx)}\norm{\uu}^2 \leq \zeta^2 + Z^2 \norm{\nabla f_{\cU(\xx)}(\xx)}^2\,, \qquad \forall \xx \in \R^d \,. \label{def:zeta}
\end{align}
\end{assumption}
This Assumption is modeled similar to our Assumption~\ref{ass:sigma}. Further, setting $Z=0$, we obtain a bound on the variance of the smoothing distribution, which is valid for subgaussian variables~\cite{duchi_randomized_2012}.

We can use the above assumption to obtain bounds on variance of the perturbed gradient.
\begin{lemma}[Stochastic Approximation]\label{lemma:1}
If $f$ is $L$-smooth and Assumption~\ref{ass:zeta}, the variance is bounded as
\begin{align}
 \E_{\uu \sim \cU(\xx)} \norm{\nabla f(\xx-\uu)- \nabla f_{\cU(\xx)}(\xx)}^2 \leq 2L^2 \zeta^2 + 2L^2 Z^2 \norm{\nabla f_{\cU(\xx)}(\xx)}^2\,, \qquad \forall \xx \in \R^d\,.
\end{align}
\end{lemma}
\begin{proof}
By Jensen's inequality and smoothness
\begin{align*}
 \E_\uu \norm{\nabla f(\xx - \uu) - \nabla f_{\cU(\xx)}(\xx)}^2 &=
 \E_\uu \norm{\nabla f(\xx - \uu) - \E_{\vv \sim \cU} \nabla f_{(\xx-\vv)}(\xx)}^2\\
 &\leq \E_{\uu,\vv} \norm{\nabla f(\xx-\uu) - \nabla f(\xx-\vv)}^2 \\ &\leq L^2  \E_{\uu,\vv} \norm{\uu-\vv}^2 \leq 2L^2 \zeta^2 + 2L^2 Z^2 \norm{\nabla f_{\cU(\xx)}(\xx)}^2\,. \qedhere
\end{align*}
\end{proof}

Now, that we have defined all the terms for the smoothing distribution in Remark~\ref{rem:var_split}, we introduce a common assumption for the stochastic noise.

\begin{assumption}\label{ass:only_sigma}
For given $f\colon \R^n \to \R$,  the perturbed stochastic gradient can be expressed as 
\begin{align}
    \nabla f(\xx - \uu,\bxi) = \nabla f(\xx - \uu) + \ww
\end{align}
where $\ww \sim \cW(\xx)$ and $\cW(\xx)$ denotes the zero-mean noise distribution, and there exist constants ($\sigma^2 >0, M >0$), such its variance can be bounded as
\begin{align}
 \E_{\ww \sim \cW(\xx)}\norm{\ww}^2 \leq \sigma^2 + M \norm{\nabla f_{\cU(\xx)}(\xx)}^2\,, \qquad \forall \xx \in \R^d \,. \label{def:only_sigma}
\end{align}
\end{assumption}
Now, we are ready to present the complete the proof for Remark~\ref{rem:var_split}. We first present its extended version as a Lemma below and then prove it.
\begin{lemma}[Extension of Remark~\ref{rem:var_split}]
If $f$ is $L$-smooth, Assumptions~\ref{ass:zeta} and~\ref{ass:only_sigma} are satisfied, and the noise ($\cW(\xx)$) and smoothing distributions ($\cU(\xx)$) are independent for $\xx$, then,
\begin{align}
\E_{\uu,\bxi} \norm{\nabla f(\xx-\uu,\bxi)-\nabla f_{\cU(\xx)}(\xx)}^2
\leq (\sigma^2 + 2(L\zeta)^2) + (M + 2(L Z)^2) \norm{ \nabla f_{\cU(\xx)} (\xx) }^2\,. \label{def:noise_complete}
\end{align}
\end{lemma}
Note that this is identical to Assumption~\ref{ass:sigma}, with $\sigma'^2 = \sigma^2 + 2(L\zeta)^2$ and $M' = M + 2(LZ)^2$.
\begin{proof}
Consider the term on the left hand side, 
\begin{align*}
\E_{\uu,\bxi} \norm{\nabla f(\xx-\uu,\bxi)-\nabla f_{\cU(\xx)}(\xx)}^2
&= \E_{\ww,\bxi} \norm{\nabla f(\xx-\uu) + \ww -\nabla f_{\cU(\xx)}(\xx)}^2\\
&= \E_{\bxi} \norm{\nabla f(\xx-\uu)-\nabla f_{\cU(\xx)}(\xx)}^2 + \E_{\ww}\norm{\ww}^2\\
&\leq (\sigma^2 + 2(L\zeta)^2) + (M + 2(L Z)^2) \norm{ \nabla f_{\cU(\xx)} (\xx) }^2 \,.
\end{align*}
The first step is obtained by applying Assumption~\ref{ass:only_sigma} to separate $\ww$. We can then separate terms of $\uu$ and $\ww$ since their distributions are independent. Then, we use Lemma~\ref{lemma:1} and Assumption~\ref{ass:only_sigma} to bound the two variance terms.
\end{proof}

\subsection{Additional details about Connection to SGD}
In this section, we provide missing proof for Lemma~\ref{lem:eq_expectation} and clarifications about Figure~\ref{fig:b}.

\subsubsection{Proof for Lemma~\ref{lem:eq_expectation}}
Consider the term $g(\Bar{\yy_t})  - g^\star $.
\begin{align*}
    g(\Bar{\yy_t}) - g^\star &\leq \frac{1}{W_T} \sum_{t=0}^T \ww_t(g(\E[\yy_t]) - g^\star)\\
    &\leq \frac{1}{W_T} \sum_{t=0}^T \ww_t(g(\E[\zz_t]) - g^\star)\\
    &\leq \frac{1}{W_T} \sum_{t=0}^T \ww_t(\E[g(\zz_t)] - g^\star)\\
    &\leq \frac{1}{W_T} \sum_{t=0}^T \ww_t \cG_t \,.
\end{align*}
For the first step, we use convexity of $g$ with coefficients $\big\{\frac{\ww_t}{W_T}\big\}_{t=0}^T$. The second step is obtained from equality in expectation. The third step is obtained from Jensen's inequality on convex $g$ and the last term is the definition of $\cG_t$.

    \subsubsection{Clarification about Figure~\ref{fig:b}}
        We would like to clarify that the objective function for Figure~\ref{fig:b} is of the form  $f(\xx) = \frac{1}{n}\sum_{i=1}^n f_i(\xx)$, where $n$ is the number of datapoints and $f_i(\xx)$ is the cross-entropy loss for the $n^{th}$ datapoint. For SGD, we sample 1 datapoint from the dataset at each step, while for the smoothing distribution, we use the formulation in~\eqref{eq:finite_sum_full_batch}, as described above.

\section{Deferred Proofs}
\label{sec:conv_proof}
In this section we provide the proofs for the convergence results in Section~\ref{sec:main_results}.

First, we state and prove an intermediate lemma for sufficient decrease which resembles~\eqref{eq:psgd_one_step}. Using this Lemma, we can easily prove the corresponding theorems for gradient noise, P\L{} and strongly-convex functions. Additionally, we restate the complete theorems for these cases which contain all the details about step sizes and exact convergence rate. 
\subsection{One Step Progress}
\begin{lemma}[One Step Progress]
    \label{lem:1_step_pl_wo}
    Let $f$ satisfy Assumptions~\ref{ass:sigma} and~\ref{ass:structure} and,
assume $g$ to be $L_g$-smooth and $\xx_t$ generated according to Algorithm $\ref{alg:smooth_sgd}$. Then, for $\gamma  \leq \frac{1}{L_g(M'+1)}$, it holds
\begin{align*}
    \frac{(1-m)}{2}\E[\norm{\nabla g(\xx_t)}^2] \leq \frac{\cG_t - \cG_{t+1}}{\gamma} + \frac{\Delta}{2} + \frac{\gamma L_g}{2}\sigma'^2 \,,
\end{align*}
where $\cG_t$ is as defined before.
\end{lemma}
\begin{proof}
Using $L_g$-smoothness of $g$, we can write
\begin{align*}
    g(\xx_{t+1}) \leq & g(\xx_{t}) + \lin{\nabla g(\xx_t) , \xx_{t+1} - \xx_t} + \frac{L_g}{2}\norm{\xx_{t+1} - \xx_t}^2\\
    \leq & g(\xx_{t}) -\gamma \lin{\nabla g(\xx_t) ,\nabla g(\xx_t - \uu_t,\bxi_t) + \nabla h(\xx_t - \uu_t,\bxi_t) } \\&\hspace{1mm} + \frac{\gamma^2 L_g}{2}\norm{\nabla g(\xx_t - \uu_t,\bxi_t) + \nabla h(\xx_t - \uu_t,\bxi_t)}^2 \,.
\end{align*}
Taking expectation wrt $\bxi_t$ and $\uu_t$, and using the inequality $\E[\norm{X}^2] = \E[\norm{X - \E[X]}^2] + \norm{\E[X]}^2$, and using the definition of smoothness we get 
\begin{align*}
    \E_{\bxi_t,\uu_t}[g(\xx_{t+1})] \leq & g(\xx_{t}) -\gamma \lin{\nabla g(\xx_t) ,\E_{\bxi_t,\uu_t}[\nabla g(\xx_t - \uu_t,\bxi_t) + \nabla h(\xx_t - \uu_t,\bxi_t)] } \\&\quad+ \frac{\gamma^2 L_g}{2}\E_{\bxi_t,\uu_t}\norm{\nabla g(\xx_t - \uu_t,\bxi_t) + \nabla h(\xx_t - \uu_t,\bxi_t)}^2\\
    \leq & g(\xx_{t}) -\gamma \lin{\nabla g(\xx_t),\nabla g_{\cU(\xx_t)}(\xx_t) + \nabla h_{\cU(\xx_t)}(\xx_t)}+ \frac{\gamma^2 L_g}{2}\norm{\nabla f_{\cU(\xx_t)}(\xx_t)}^2\\&\quad+  \frac{\gamma^2 L_g}{2}\E_{\bxi_t,\uu_t}[\norm{\nabla f(\xx_t - \uu_t,\bxi_t) -  \nabla f_{\cU(\xx_t)}(\xx_t)}^2] \,.
\end{align*}
Using Assumption~\ref{ass:sigma}, with $\gamma \leq \frac{1}{L_g(M' + 1)}$
\begin{align}
\begin{split}
    \E_{\bxi_t,\uu_t}[g(\xx_{t+1})]     \leq & g(\xx_{t}) -\gamma \lin{\nabla g(\xx_t),\nabla g_{\cU(\xx_t)}(\xx_t) + \nabla h_{\cU(\xx_t)}(\xx_t)}\\&\quad+ \frac{\gamma^2 L_g(M'+1)}{2}\norm{\nabla f_{\cU(\xx_t)}(\xx_t)}^2  +\frac{\gamma^2 L_g}{2}\sigma'^2 \,. \label{eq:common_step}
\end{split} \\
\begin{split}
    \E_{\bxi_t,\uu_t}[g(\xx_{t+1})]\leq & -\frac{\gamma}{2}\bigg(\norm{\nabla g(\xx_t)}^2 -\norm{\nabla h_{\cU(\xx_t)}(\xx_t) + g_{\cU(\xx_t)}(\xx_t) - g(\xx_t)}^2\bigg)\\&\quad +g(\xx_{t}) +\frac{\gamma^2 L_g}{2}\sigma'^2 \,.
\end{split} \notag
\end{align}
Now, using Assumption~\ref{ass:structure}.
\begin{align*}
\E_{\bxi_t,\uu_t}[g(\xx_{t+1})] \leq & g(\xx_{t}) -\frac{\gamma(1-m)}{2}\norm{\nabla g(\xx_t)}^2 + \frac{\gamma\Delta}{2} +\frac{\gamma^2 L_g}{2}\sigma'^2  \,.
\end{align*}
Taking full expectation on both sides and subtracting $\min_{\xx\in \R^d}g(x)$ from both sides, we get the required result.
\end{proof}
\subsection{Gradient Norm Convergence (Proof of Theorem~\ref{thm:grad_norm_wo})}
We first state the extended version of Theorem~\ref{thm:grad_norm_wo}.
\begin{extended_theorem} (Gradient Norm convergence) 
 Under the assumptions in Lemma~\ref{lem:1_step_pl_wo}, for stepsize $\gamma  \leq \frac{1}{L_g(M' +1)}$, after running the Algorithm~\ref{alg:smooth_sgd} for T steps, it holds:
 \begin{align*}
     \Phi_T \leq \bigg(\frac{2 \cG_0}   {T\gamma(1-m)} + \frac{\gamma L_g\sigma'^2}{1-m}\bigg) + \frac{\Delta}{1-m} \,,
 \end{align*}
where $\Phi_T=\frac{1}{T}\sum_{t=0}^{T-1}\E[\norm{\nabla g(\xx_t)}^2]$.\\
Further, for $\epsilon>0$ and $\gamma = \min\{\frac{1}{L_g(M'+1)}, \frac{\epsilon(1-m) + \Delta}{2L_g\sigma'^2}\}$, then
\begin{align*}
    T = \cO\bigg(\frac{M'+1}{\epsilon(1-m) + \Delta} + \frac{\sigma'^2}{\epsilon^2(1-m)^2 + \Delta^2}\bigg)L_g \cG_0
\end{align*}
iterations are sufficient to obtain $\Phi_T = \cO(\epsilon + \frac{\Delta}{1-m})$
\end{extended_theorem}
\begin{proof}
We can sum the terms of Lemma~\ref{lem:1_step_pl_wo} for $t=0$ to $T-1$, and divide both sides by $T$,to obtain 

\begin{align*}
\frac{1}{T}\sum_{t=0}^{T-1}\E[\norm{\nabla g(\xx_t)}^2] \leq \frac{2(\cG_0 - \cG_{T})}{T\gamma(1-m)} + \frac{\Delta}{(1-m)} + \frac{\gamma L_g}{(1-m)}\sigma'^2 \,,
\end{align*}
This proves the first part of the above Theorem. We can choose step sizes according to obtain rates in terms of $\epsilon$. This can be found in \cite[Lemma~3]{ajalloeian_analysis_2020}  and \cite[Theorem~4]{ajalloeian_analysis_2020} with different constants and notation.
\end{proof}

\subsection{Convergence for P\L{} functions (Proof of Theorem~\ref{thm:pl_wo})}
We state the extended version of Theorem~\ref{thm:pl_wo}.
\begin{extended_theorem}
Under Assumptions of Lemma~\ref{lem:1_step_pl_wo} and the additional assumption that $g$ is $\mu_g$-P\L{}, it holds for any stepsize $\gamma \leq \frac{1}{L_g(M'+1)}$,
 \begin{align*}
     & \cG_T \leq (1 - \gamma \mu_g (1-m))^T \cG_0 + \frac{1}{2}\Xi\,,
     & \qquad \text{ where } & & \Xi = \frac{\Delta}{\mu_g(1-m)} + \frac{\gamma L_g \sigma'^2 }{\mu_g(1-m)}\,.
 \end{align*}
 Further, by choosing $\gamma = \min\{\frac{1}{L_g(M' +1)}, \frac{\epsilon(1-m)\mu_g + \Delta}{L_g\sigma'^2  }\}$, for any $\epsilon > 0$, 
 \begin{align*}
     T = \Tilde{\cO}\bigg((M' + 1)\log\frac{1}{\epsilon} + \frac{\sigma'^2}{\epsilon(1-m)\mu_g + \Delta}\bigg)\frac{\kappa}{1-m}
 \end{align*}
 iterations are sufficient to obtain $\cG_T = \cO(\epsilon + \frac{\Delta}{\mu_g(1-m)})$, where $\kappa := \frac{L_g}{\mu_g}$ and $\Tilde{\cO}$ hides only log terms. 
\end{extended_theorem}
\begin{proof}
We use the P\L{} condition in Lemma~\ref{lem:1_step_pl_wo}, to obtain
\begin{align*}
    \mu_g\cG_t \leq& \frac{(\cG_t - \cG_{t+1})}{\gamma(1-m)} + \frac{\Delta}{2(1-m)} + \frac{\gamma L_g}{2(1-m)}\sigma'^2 \\
    \cG_{t+1} \leq& (1 - \mu_g\gamma(1-m))\cG_t  + \frac{\Delta\gamma}{2} + \frac{\gamma^2 L_g}{2}\sigma'^2
\end{align*}
Unfolding the above recursion from $t=0$ to $t=T-1$, we get the first part of above Theorem. For the convergence rates in terms of $\epsilon$, we can choose step size $\gamma$ accordingly. This is similar to \cite[Theorem~6]{ajalloeian_analysis_2020} with different constants and notation.
\end{proof}

\subsection{Convergence for Strongly-convex functions (Proof of Theorem~\ref{thm:str_cvx_wo})}
\label{sec:strongconvexproof}

We first state the extended version of Theorem~\ref{thm:str_cvx_wo}.
\begin{extended_theorem}
 Under Assumptions~\ref{ass:sigma} and \ref{ass:structure_m_str_cvx_wo}, and if $g$ is $\mu_g$-strongly convex, running Algorithm~\ref{alg:smooth_sgd} for T steps, with $\gamma \leq \frac{1 - \sqrt{m}}{L_g(1+\sqrt{m})^2(M' + 1)}$, there exist non-negative weights $\{w_t\}_{t=0}^T$, with $W_T = \sum_{t=0}^T w_t$, such that 
 \begin{align*}
     \frac{1}{W_T}\sum_{t=0}^T w_t \cG_T + \frac{\mu_g}{2}  d_{T+1} = \cO\bigg(\frac{ d_0}{\gamma(1-\sqrt{m})} \exp\bigg(-\frac{(1-\sqrt{m})\gamma\mu_g T}{2}\bigg) + \Xi \bigg)
 \end{align*}
where $ \cG_t$ is same as defined  before, $d_t=\E[\norm{\xx_t - \xx_g^\star}^2] $ , $\xx_g^\star = \argmin_{\xx \in\R^d} g(\xx)$, and 
\begin{align*}
    \Xi = \frac{\gamma(\sigma'^2 +\Delta(M'+ 1)) }{(1-\sqrt{m})}+ \frac{2\Delta}{\mu_g(1-\sqrt{m})^2} \,.
\end{align*}

Further, choosing $\gamma =  \min\bigg\{\frac{(1 - \sqrt{m})}{L_g (M'+1) (1 + \sqrt{m})^2}, \frac{\mu_g\epsilon(1- \sqrt{m})^2 + 4\Delta}{2(\sigma'^2 +\Delta(M'+ 1))(1 - \sqrt{m})\mu_g}\bigg\}$, 
\begin{align*}
    T = \Tilde{\cO}\bigg(\frac{2\kappa (M' + 1 )(1 + \sqrt{m})^2}{(1-\sqrt{m})^2}\log\frac{1}{\epsilon} + \frac{4(\sigma'^2+\Delta(M' + 1))}{\mu_g\epsilon(1-\sqrt{m})^2 + 4\Delta}\bigg)
\end{align*}
iterations are sufficient to obtain $\frac{1}{W_T}\sum_{t=0}^T w_t \cG_T = \cO(\epsilon + \frac{4\Delta}{\mu_g(1-\sqrt{m})^2})$.
\end{extended_theorem}
\begin{proof}
Consider $\norm{\xx_t - \xx_{g}^\star}^2$, and take expectations with respect to $\uu_t, \bxi_t$, on both sides, further use $\E[\norm{X}^2] = \E[\norm{X - \E[X]}^2] + \norm{\E[X]}^2$ and Assumption~\ref{ass:sigma}.
\begin{align}
    \norm{\xx_{t+1} - \xx_{g}^\star}^2 &= \norm{\xx_{t} - \xx_{g}^\star}^2 - 2\gamma \lin{\nabla f(\xx_t - \uu_t,\bxi_t), \xx_t - \xx_{g}^\star} \nonumber\\&\quad+ \gamma^2 \norm{\nabla f(\xx_t - \uu_t,\bxi_t)}^2 \notag \\
    \E_{\uu_t,\bxi_t}[\norm{\xx_{t+1} - \xx_{g}^\star}^2]
    &= \norm{\xx_{t} - \xx_{g}^\star}^2 - 2\gamma \lin{\nabla f_{\cU(\xx_t)}(\xx_t), \xx_t - \xx_{g}^\star} +\gamma^2\norm{\nabla f_{\cU(\xx_t)}(\xx_t)}^2 \nonumber\\&\quad
    +\gamma^2 \E_{\uu_t,\bxi_t}[ \norm{\nabla f(\xx_t - \uu_t,\bxi_t) -\nabla f_{\cU(\xx_t)}(\xx_t)}^2] \notag \\
&\leq \norm{\xx_{t} - \xx_{g}^\star}^2 - 2\gamma \lin{\nabla g(\xx_t), \xx_t - \xx_{g}^\star}    +\gamma^2 \sigma'^2 
\nonumber\\&\quad- 2\gamma \lin{\nabla g_{\cU(\xx_t)}(\xx_t) + \nabla h_{\cU(\xx_t)}(\xx_t) - \nabla g(\xx_t), \xx_t - \xx_{g}^\star}  \nonumber\\&\quad+\gamma^2(M'+ 1)\norm{\nabla g_{\cU(\xx_t)}(\xx_t) + \nabla h_{\cU(\xx_t)}(\xx_t)}^2 \,. \label{eq:str_cvx_intermediate}
\end{align}
Let $\widehat{\nabla g(\xx_t)}$ and $\widehat{\nabla g(\xx_t)}_{\perp}$ be the units vector in direction of $\nabla g(\xx_t)$ and perpendicular to it, respectively. For clarity of notations, let $(\nabla h_{\cU(\xx_t)}(\xx_t) + \nabla g_{\cU(\xx_t)}(\xx_t) - \nabla  g(\xx_t)) = \rr(\xx_t)$.
First, we bound the component perpendicular to $\nabla g(\xx_t)$, using Assumption~\ref{ass:structure}
\begin{align}
   (\rr(\xx_t))_{g_\perp}\lin{\widehat{\nabla g(\xx_t)}_{\perp}, \xx_t - \xx_{g}^\star}\nonumber &\leq \frac{\mu_g}{4}\norm{\xx_t - \xx_{g}^\star}^2  +  \frac{1}{\mu_g}\abs{(\rr(\xx_t))_{g_\perp}}^2\\
&\leq   \frac{\mu_g (1-\sqrt{m})}{4}\norm{\xx_t - \xx_{g}^\star}^2  +  \frac{\Delta}{\mu_g(1-\sqrt{m})} \,. \label{eq:h_perp}
\end{align}
Now, consider the component along $\nabla g(\xx_t)$ and strong convexity of $g$ implies $\lin{\nabla g(\xx_t),\xx_t - \xx_{g}^\star} \geq 0 $, and using Assumption~\ref{ass:structure}
\begin{align}
    (\rr(\xx_t))_{g}\lin{\widehat{\nabla g(\xx_t)}_{g}, \xx_t - \xx_{g}^\star}
    &\geq -\frac{\abs{(\rr(\xx_t))_{g}}}{\norm{\nabla g(\xx_t)}}\lin{\nabla g(\xx_t), \xx_t - \xx_{g}^\star} \notag \\
    &\geq -\sqrt{m}\lin{\nabla g(\xx_t), \xx_t - \xx_{g}^\star} \,. \label{eq:h_along_g} 
\end{align}
Additionally, consider $\norm{\nabla g_{\cU(\xx_t)}(\xx_t) + \nabla h_{\cU(\xx_t)}(\xx_t)}^2$ and use Assumption~\ref{ass:structure_m_str_cvx_wo}. 
\begin{align}
    \norm{\nabla g_{\cU(\xx_t)}(\xx_t) + \nabla h_{\cU(\xx_t)}(\xx_t)}^2 
    &\leq \norm{\nabla g_{\cU(\xx_t)}(\xx_t) + \nabla h_{\cU(\xx_t)}(\xx_t) - \nabla g(\xx_t) + \nabla g(\xx_t)}^2 \notag \\
    &\leq \norm{\vv + \nabla g(\xx_t)}^2  \notag \\
    &\leq \norm{(\vv)_g \widehat{\nabla g(\xx_t)} + (\vv)_{g_{\perp}} \widehat{\nabla g(\xx_t)}_{\perp} + \nabla g(\xx_t)}^2  \notag \\
    &\leq \abs{(\vv)_g  + \norm{\nabla g(\xx_t}}^2 + \abs{(\vv)_{g_{\perp}}}^2  \notag \\
    &\leq (1 + \sqrt{m})^2\norm{\nabla g(\xx_t)}^2 + \Delta  \,. \label{eq:norm_bound}
\end{align}
Using Eqns.~\eqref{eq:h_perp}, \eqref{eq:h_along_g} and~\eqref{eq:norm_bound} in Eq.~\eqref{eq:str_cvx_intermediate}, we get 
\begin{align*}
    \E_{\uu_t,\bxi_t}[\norm{\xx_{t+1} - \xx_{g}^\star}^2] &\leq \norm{\xx_{t} - \xx_{g}^\star}^2(1 +\frac{\gamma\mu_g (1-\sqrt{m})}{2})  - 2\gamma(1-\sqrt{m}) \lin{\nabla g(\xx_t), \xx_t - \xx_{g}^\star}    
\\&\quad+\gamma^2(M' + 1)(1 + \sqrt{m})^2\norm{\nabla g(\xx_t)}^2+\gamma^2 (\sigma'^2 + \Delta(M' + 1))\\&\quad  + \frac{2\gamma \Delta}{\mu_g(1 - \sqrt{m})} \,.
\end{align*}
 Now, using strong-convexity and smoothness of $g$, we get
 \begin{align*}
    \E_{\uu_t,\bxi_t}[\norm{\xx_{t+1} - \xx_{g}^\star}^2] &\leq \norm{\xx_{t} - \xx_{g}^\star}^2\bigg(1 - \frac{\gamma\mu_g (1-\sqrt{m})}{2}\bigg)+\gamma^2 (\sigma'^2 + \Delta(M' + 1))+ \frac{2\gamma \Delta}{\mu_g (1 - \sqrt{m})}\\&\quad- 2\gamma(1-\sqrt{m})\bigg(1 - \frac{\gamma L_g(M' + 1)(1 + \sqrt{m})^2}{2(1 - \sqrt{m})}\bigg) (g(\xx_t) - g(\xx_g^\star)) \,.
\end{align*}
Now, taking $\gamma \leq \frac{(1 - \sqrt{m})}{L_g(M' + 1)(1 + \sqrt{m})^2}$, taking complete expectations, and substituting $\cG_t = \E[g(\xx_t)] - g(\xx_g^\star)$ and $d_t = \E[\norm{\xx_t - \xx_g^\star}^2]$.
 \begin{align*}
    d_{t+1} &\leq d_t\bigg(1 - \frac{\gamma\mu_g (1-\sqrt{m})}{2}\bigg)+\gamma^2 (\sigma'^2 + \Delta(M' + 1))+ \frac{2\gamma \Delta}{\mu_g(1 - \sqrt{m})}\\&\quad- \gamma(1-\sqrt{m}) \cG_t \,.
\end{align*}

We follow analysis in \cite[Lemma~2]{stich_unified_2019} to multiply both sides by $w_t = \bigg(1 - \frac{\gamma\mu_g (1-\sqrt{m})}{2}\bigg)^{-(t+1)}$. If $\frac{\gamma\mu_g (1-\sqrt{m})}{2} < 1$, we sum over $t=0$ to $T$ and divide both sides by $W_T = \sum_{t=0}^T w_t$. We obtain the following results after performing these steps, 
\begin{align*}
    \frac{(1-\sqrt{m})}{W_T}\sum_{t=0}^T w_t\cG_t + \frac{w_T d_{T+1}}{\gamma W_T} \leq \frac{d_t}{\gamma W_T} + \frac{2\Delta}{\mu_g(1 - \sqrt{m})} + \gamma(\sigma'^2+ \Delta(M' + 1))  \,.
\end{align*}
Since $W_T \leq \frac{w_T}{(\gamma\mu_g(1 - \sqrt{m})/2)\gamma}$ and $W_T \geq w_T$, we obtain the first inequality
\begin{align*}
    \frac{1}{W_T}\sum_{t=0}^T w_t\cG_t + \frac{\mu_g}{2}\frac{d_{T+1}}{\gamma W_T} \leq& \frac{d_0}{\gamma(1-\sqrt{m})}\exp\bigg(-\frac{\mu_g \gamma (1-\sqrt{m})T}{2}\bigg)+\frac{2\gamma \Delta}{\mu_g(1 - \sqrt{m})^2}\\&\qquad + \frac{\gamma(\sigma'^2 +\Delta(M' + 1))}{(1-\sqrt{m})} \,.
\end{align*}
For the second part, first let $\alpha = \sigma'^2+ \Delta(M' + 1)$ and $\beta = M' + 1$
Then, we denote the RHS of the main convergence result in terms of $\gamma$ and $T$.
\begin{align*}
    \Theta(\gamma,T) = \frac{d_0}{\gamma(1-\sqrt{m})}\exp\bigg(-\frac{\mu_g \gamma (1-\sqrt{m})T}{2}\bigg) + \frac{\alpha}{(1-\sqrt{m}} + \frac{2\Delta}{\mu_g(1 - \sqrt{m})^2} \,.
\end{align*}
We show that our bound for $\Theta(\gamma,T) = \cO\bigl(\epsilon + \frac{4\Delta}{\mu_g(1- \sqrt{m})^2}\bigr)$ is achieved by $\gamma = \min\{\gamma_1,\gamma_2\}$ and $T = \max\{T_1, T_2\}$
\begin{align*}
 \gamma_1 =& \frac{(1 - \sqrt{m})}{L_g M' (1 + \sqrt{m})^2},\qquad \gamma_2 = \frac{\mu_g\epsilon(1- \sqrt{m})^2 + 4\Delta}{2\alpha(1 - \sqrt{m})\mu_g}\\
 T_1 &= \frac{2\beta L_g(1 + \sqrt{m})^2}{\mu_g(1-\sqrt{m})^2}\log\bigg(\frac{2L_g\beta d_0 ( 1 + \sqrt{m})^2}{\epsilon(1-\sqrt{m})^2}\bigg),\\
 T_2 &=\frac{4\beta}{\mu_g\epsilon(1-\sqrt{m})^2 + 4\Delta}\log\bigg(\frac{4d_0 \alpha\mu_g}{(\mu_g \epsilon(1 - \sqrt{m})^2 + 4\Delta)\epsilon}\bigg) \,.
\end{align*}

If $\gamma = \gamma_1$, then $\frac{\gamma \alpha}{(1 - \sqrt{m})} \leq \frac{\epsilon}{2} + \frac{2\Delta}{\mu_g(1-\sqrt{m})^2}$. Then, we can choose $T \geq T_1$, so that $\Theta(\gamma, T) \leq \epsilon + \frac{4\Delta}{\mu_g(1-\sqrt{m})^2}$

Similarly, if $\gamma = \gamma_2$, then $\frac{\gamma \alpha}{(1 - \sqrt{m})} \leq \frac{\epsilon}{2} + \frac{2\Delta}{\mu_g(1-\sqrt{m})^2}$. Then, we can choose $T \geq T_2$, so that $\Theta(\gamma, T) \leq \epsilon + \frac{4\Delta}{\mu_g(1-\sqrt{m})^2}$.
\end{proof}

\subsection{Additional Settings}
In this subsection, we present alternative formulations to our Assumptions, namely, for bounded non-convexity $h$ and for exact smooth oracle $\nabla f_{\cU(\xx)}(\xx)$, instead of the perturbed gradient.
\subsubsection{Convergence for Exact Smooth Oracle $\nabla f_{\cU(\xx)}(\xx)$}\label{sec:full_smooth_grad}
While we have derived all results assuming we have access to $\nabla f(\xx + \uu;\bxi)$, our results can be extended to the case when we have access to $\nabla f_{\cU(\xx)}(\xx;\bxi)$. This extension is similar to extensions of SGD results to GD. This is  done by setting the variance of gradients to 0, by setting $\sigma^2 = M =0$. Similarly, for our case setting $\zeta^2 = Z = 0$, yields converge rates with gradient oracle $\nabla f_{\cU(\xx)}$. This does not mean that the smoothing distribution $\cU(\xx)$ has $0$ variance, just that the contribution to gradient noise due to smoothing is $0$, again motivating the connection between smoothing and SGD.

\subsubsection{Non-convexity $h$ with Bounded Gradients }\label{sec:bdd_gradients}
In this section, we explore a class of non-convex functions satisfying our formulation~\eqref{eq:structure}, but which are easy to solve. Consider as before that $g(\xx)$ and $h(\xx)$ denote the convex part and non-convex perturbation of $f(\xx)$, respectively.  We now provide a few definitions which we will use later.

A point $\xx \in \R^d$ is a stationary point of a differentiable function $f:\R^d \to \R$ if
\begin{align*}
    \nabla f(\xx) = 0 \,.
\end{align*}
Let $\cX^\star$ denote the set of stationary points of $f$. 
Additionally, let $g^\star = \min_{\xx\in \R^d}g(\xx)$.

A function $h: \R^d \to \R$ has $B_2$-bounded gradients if 
\begin{align}
    \norm{\nabla h(\xx)}^2 \leq B_2 \qquad \forall \xx \in \R^d \,.\label{def:h_grad_bdd}
\end{align}

A function $h: \R^d \to \R$ is $B_1$-bounded if 
\begin{align}
    \abs{h(\xx)} \leq B_1 \qquad \forall \xx \in \R^d \,. \label{def:h_bdd}
\end{align}

With these definitions, we provide the below lemma, which illustrates the impact of a simple (bounded and gradient bounded) $h$ on the stationary points of $f$.
\begin{lemma}\label{lem:f_bdd}
Let $f$ satisfy structure~\eqref{eq:structure} with convex part $g$ and non-convex part $h$. 
\begin{itemize}
\item If $g$  is $\mu_g$-P\L{}  and $h$ is $B_2$-gradient bounded 
\begin{align*}
    g^\star \leq g(\xx) \leq g^\star + \frac{B_2}{2\mu_g}\,, \qquad \forall \xx \in \mathcal{X}^\star \,.
\end{align*}
\item If $g$ is $\mu_g$-strongly convex and $h$ is $B_2$-gradient-bounded
\begin{align*}
    \norm{\xx - \xx_g^\star}^2 \leq \frac{B_2}{\mu_g^2}\,, \qquad\forall \xx \in \mathcal{X}^\star\,.
\end{align*}
\item If $g$ is $\mu_g$-P\L{} and $h$ is $B_1$-bounded and $B_2$-gradient bounded 
\begin{align*}
\begin{split}
    g^\star - B_1 \leq f(\xx) &\leq g^\star + B_1 + \frac{B_2}{2\mu_g} \,, \qquad \forall \xx \in \mathcal{X}^\star \,,\\
    \abs{f(\xx) - f(\yy)} &\leq 2B_1 + \frac{B_2}{2\mu_g}\,, \qquad \forall \xx, \yy \in \mathcal{X}^\star\,.
\end{split}
\end{align*}
\end{itemize}
\end{lemma}
\begin{proof}
Let $\yy$ be a stationary point of $f$. Then,
\begin{align*}
    \nabla g(\yy) = - \nabla h(\yy) \,.
\end{align*}
    For the first part, since $g$ is  and $h$ is $B_2$-gradient bounded,
\begin{align*}
    2\mu_g(g(\yy) - g^\star) &\leq \norm{\nabla g(\yy)}^2 = \norm{\nabla h(\yy)}^2\leq B_2 \,.
\end{align*}
For the second part, since $g$ is $\mu_g$-strongly convex with global minima $\xx_g^\star$
\begin{align*}
    g(\yy) \geq& g^\star + \frac{\mu_g}{2}\norm{\yy -\xx_g^\star}^2 \,,
\end{align*}
and the claim follows together with the first part of this lemma (all $\mu_g$-strongly convex functions are also $\mu_g$-P\L{}).

For the third part, assuming $h$ is $B_1$- bounded with the result from first part,
\begin{align*}
    g^\star + h(\yy)\leq g(\yy) + h(\yy) &\leq g^\star + h(\yy) + \frac{B_2}{2\mu_g} \,,\\
    g^\star - B_1\leq f(\yy) &\leq g^\star + B_1 + \frac{B_2}{2\mu_g} \,. \qedhere
\end{align*}
\end{proof}
From the above lemma, we can see that if $h$ is gradient bounded, all its stationary points are close to minima of $g$. Thus, even GD on such a function should always end up close to the global minima. Note that Assumption~\ref{ass:structure} is weaker than bounded gradients for $h$, as we allow $h$ to have unbounded gradients and its stationary points are also not constrained to a neighborhood. This is demonstrated by our toy example $f(x) = x^2 + ax\sin(bx)$, which we describe in detail in the next section.

\section{Investigating Examples}\label{sec:add_examples}
In this section, we further investigate our toy example $f(x) = x^2 +ax\sin(bx)$ and utilize it to compare our settings to other applications of non-convex smoothing in ~\cite{kleinberg_alternative_2018,hazan_graduated_2015}.
Consider $f(x) = x^2 +  a x\sin( b x)$ and $\cU = \cN(0, \zeta^2)$ as in the main text. For $g(x) = x^2$ and $h(x) = a x\sin( b x)$, we observe that 
\begin{align*}
    g_{\cU}(x) = x^2 + \zeta^2,\quad &\hspace{2mm} h_{\cU}(x) = a e^{-(b^2\zeta^2)/2}(  b \zeta^2 \cos(b x) + x\sin( b x))\\
    \nabla g_{\cU}(x) = 2x, &\quad \nabla h_{\cU}(x) = ab e^{-(b^2\zeta^2)/2}( (1 -b \zeta^2) \sin(b x) + x\cos( b x))\\
    \norm{\nabla h_{\cU}(x) + \nabla g_{\cU}(x) - \nabla g(x) }^2 &\leq  a^2b^2 e^{-b^2 \zeta^2}( x^2 + (b\zeta^2 - 1)^2)\\
    &\leq  \frac{a^2b^2 e^{-b^2 \zeta^2}}{4}(\norm{\nabla g(x)}^2  + 4(b\zeta^2 - 1)^2)
\end{align*}
To satisfy Assumption~\ref{ass:structure} we can choose $m=\frac{1}{4}a^2b^2 e^{-b^2 \zeta^2}$ or $\zeta = \frac{1}{b}\sqrt{2\ln(ab) - \ln(4m)}$ (note that $m<1$) and $\Delta = a^2b^2 e^{-b^2 \zeta^2} (b\zeta^2 - 1)^2 = \frac{4m}{b^2}(2\ln(ab) - \ln(4m) - b)^2 $.

For any finite value of $\zeta$, the function $f_\cU$ is never convex. However, for every  $\zeta > \frac{1}{b}\sqrt{(2\ln(ab) - \ln(4))}$, we can always find $m < 1,\Delta > 0$ which satisfies our Assumption $\ref{ass:structure}$.
\subsection{Toy Example is not convex after smoothing}
\label{sec:toyexampleappendix}
Consider the toy example again, $f(x) = x^2 + 10x\sin(x)$, with smoothing $f$ with $\cN(0,\zeta^2)$. We obtain:
\begin{align}
    f_{\cU}(x) &= x^2 + \zeta^2 + a e^{-(b\zeta)^2/2}\bigl(b\zeta^2 \cos(bx)+ x\sin(bx)\bigr)\,.
\end{align}
According to our structure~\eqref{eq:structure}, we can pick $g(x)= x^2$ and $h(x) = ax\sin(bx)$. We observe that smoothing reduces the non-convexity in the function and it starts resembling its convex component $g$. This is better visualized in Figure~\ref{fig:example}, where we plot the function and its gradient for parameters $a=10$ and $b=1$ and $\zeta \in \{0,1,2\}$, where $\zeta = 0$ corresponds to no smoothing.

Further, if we take our toy example again, $f(x) = x^2 + 10x\sin(x)$, we can see that even after smoothing $f$ with $\cN(0,\zeta^2)$, $f_\cU$ still has local minima and is not strongly-convex. To generate a concrete example, consider $\zeta = 2$, and denote the smoothed function with $f_{\zeta}$ which is plotted in Figure~\ref{fig:diff_smoothing}, and for better visualization additionally in Figure~\ref{fig:smooth_2}.
The smoothed function $f_2$ has two minima, close to $x\approx -2.56$ and $x \approx 2.56$ and an additional stationary point at $x=0$. Therefore, the function $f_2$ is not strongly convex on a $3\zeta$-ball around its minima (as each such ball contains also $x=0$ and the other minima). Therefore, the example function $f_2$ does not satisfy the local strong convexity condition that is required for $(c, \delta)$-nice functions, but it satisfies our Assumption~\ref{ass:structure} (note that $\zeta > 2$ satisfies the sufficient  condition derived above).

\begin{figure}[htb]
\centering
\includegraphics[width=0.4\linewidth]{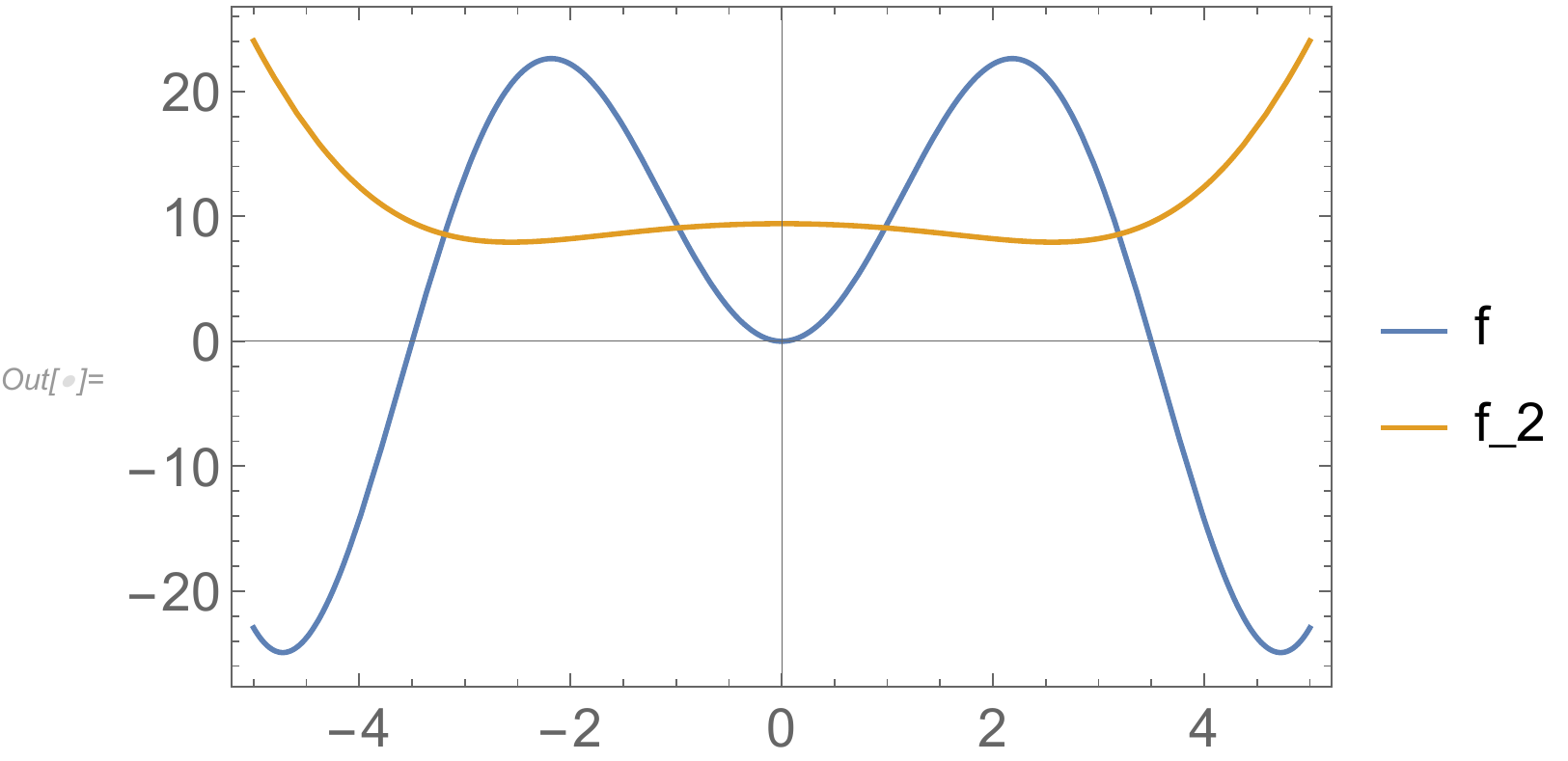}
\caption{Function $f(x)$ and $f_\zeta(x)$ for $\zeta=2$ (the same function as in Figure~\ref{fig:diff_smoothing}, highlighting that $f_2$ is not strongly convex in a $3\zeta$-ball around its minima, as required for $(c,\zeta)$-nice functions, but $f_2$ satisfies Assumption~\ref{ass:structure}.}
\label{fig:smooth_2}
\end{figure}

\subsection{Comparison to other Applications of Non-Convex Smoothing}\label{sec:comparison}
In \cite{hazan_graduated_2015}, the notion of graduated optimization is utilized, by successively smoothing with decreasing $\delta$ variance, to converge to global optima of a class of non-convex Lipschitz functions in a bounded domain $\cX$ ($(c,\delta)$-nice,  \cite[Definition~3.2]{hazan_graduated_2015}). Convergence of their method relies on the function becoming strongly-convex on $\cX$ after $c\delta$-smoothing. For a fixed domain, we can set $\zeta = c\delta > \frac{1}{b}\sqrt{(2\ln(ab) - \ln(4))}$, with appropriate $a,b$ such that our toy example is never strongly convex in a fixed interval inside $\cX$, but satisfies our Assumption~\ref{ass:structure}. Thus, their analysis fails on our example. Further, on a bounded domain, if a function is strongly-convex after smoothing, it satisfies our Assumption~\ref{ass:structure} for the same smoothing with $m=\Delta=0$. Thus, all $(c,\delta)$-nice functions also satisfy this assumption.

Our assumptions are weaker than those required in \cite{kleinberg_alternative_2018}. Notably, \cite{kleinberg_alternative_2018} consider only smoothing with bounded support, while we do not have this restriction.
Moreover, they need to assume that for given $\cU$, $f_\cU$ is star convex. We see from Figure~\ref{fig:diff_smoothing} that our toy function is not star convex for all $\zeta^2$, while our Assumption~\ref{ass:structure} holds. This shows, that our setting allows more flexibility  in the parameters.

\subsection{Comparing to $(c,\delta$)-Nice Functions~\cite{hazan_graduated_2015}}
\label{sec:hazan_functions}
We consider the toy example which is $(c, \delta)$-nice, mentioned in \cite{hazan_graduated_2015}, and show that this function can be optimized under our biased gradient assumptions as well.
Consider $\xx = (\xx_1,\xx_2,\ldots,\xx_d) \in \R^d$
\begin{equation*}
    f(\xx) = 0.5 \norm{\xx}^2 - \alpha e^{-\frac{x_1 - 1}{2\lambda^2}}
\end{equation*}
This function is $(\sqrt{d},0.5)$-nice for $\lambda \leq 0.1$ and $\alpha \in \big[0, \frac{1}{200}\big]$. Note that, if we consider $g(\xx) = \xx^2$ and $h(\xx) = - \alpha e^{-\frac{x_1 - 1}{2\lambda^2}}$, after smoothing with $\cU = \cN(\0,\zeta^2 I_d)$, we obtain --
\begin{align*}
    \norm{\nabla h_{\cU}(\xx) + \nabla g_{\cU}(\xx) - \nabla g(\xx)}^2 \leq & \frac{\alpha^2\zeta^4}{\lambda^2(\zeta^2 + \lambda^2)^3}\bigl(\norm{\nabla g(\xx)}^2 + 1\bigr) \,.
\end{align*}
Here, choosing $\zeta = k\lambda$, this function satisfies Assumption~\ref{ass:structure} with  $\Delta = m = \frac{\alpha^2 k^4}{(k^2+1)^3\lambda^4}$. For every valid $\alpha , \lambda$, we can choose $k$ such that $m < 1$.

\subsection{Additional experiments on toy example}\label{sec:add_exp}
We perform additional experiments on our toy example for the same settings as Section~\ref{sec:experiments}. We implement Perturbed SGD with no gradient noise and different smoothing  by controlling $\zeta$ and SGD, with a Gaussian gradient noise distribution, $\cW = \cN(0,\sigma^2)$. 
\begin{figure}[th!]
     \vspace{-2mm}
     \centering
     \subfigure[Low noise $\zeta = 0.5$\label{fig:smooth_gd_0.5}
]{
\includegraphics[width=0.2\textwidth]{smoothgd_0.5.jpg}
         \includegraphics[width=0.2\textwidth]{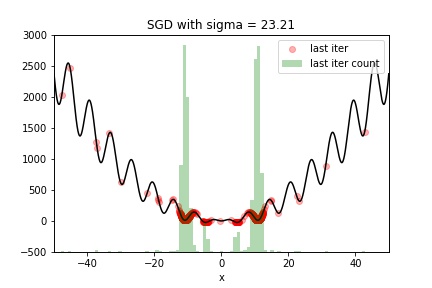}
}
     \subfigure[Low noise $\zeta = 1$\label{fig:smooth_gd_21}]{
     \includegraphics[width=0.2\textwidth]{smoothgd_1.jpg}                  \includegraphics[width=0.2\textwidth]{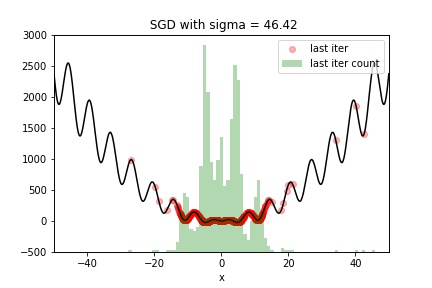}

         }
         \subfigure[Intermediate noise $\zeta = 2$\label{fig:smooth_gd2}]{
         \includegraphics[width=0.2\textwidth]{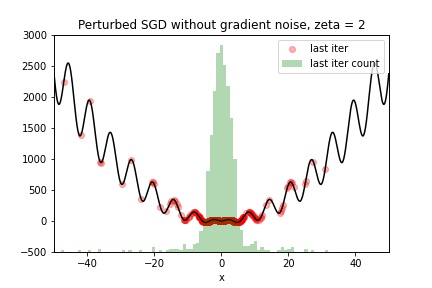}
         \includegraphics[width=0.2\textwidth]{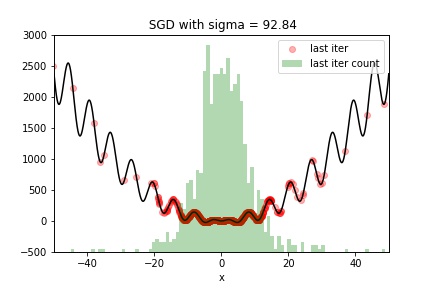}         
         }
     \subfigure[Intermediate noise $\zeta = 5$\label{fig:smooth_gd_5}]{
         \includegraphics[width=0.2\textwidth]{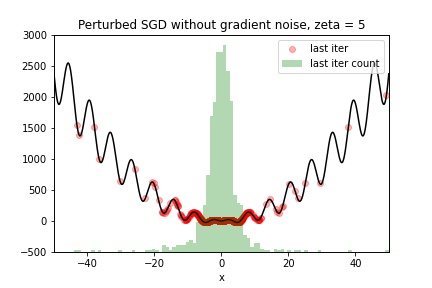}
         \includegraphics[width=0.2\textwidth]{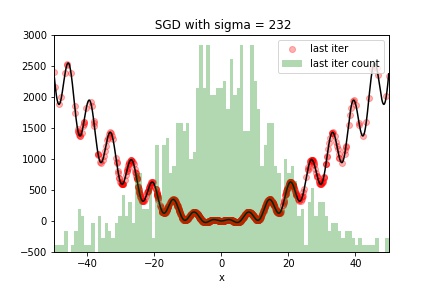}
         }
     \subfigure[High noise $\zeta = 10$\label{fig:smooth_gd_10}]{
         \includegraphics[width=0.2\textwidth]{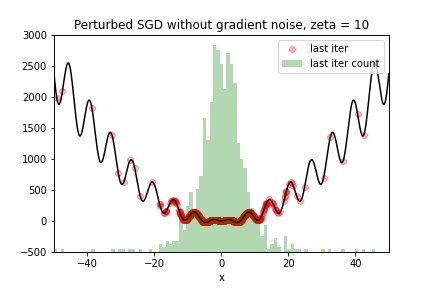}
        \includegraphics[width=0.2\textwidth]{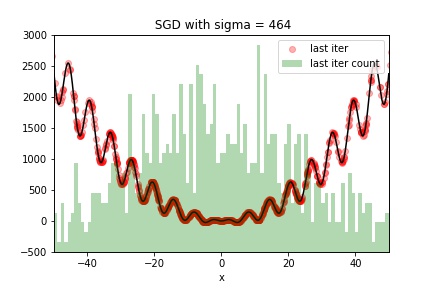}
        }
     \subfigure[High Noise $\zeta = 20$\label{fig:smooth_gd_20}]{
         \includegraphics[width=0.2\textwidth]{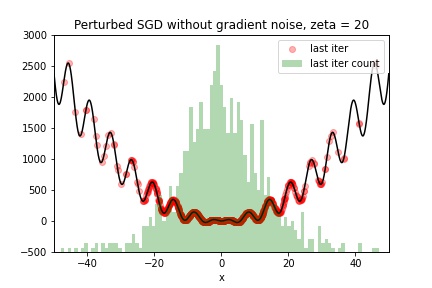}
         \includegraphics[width=0.2\textwidth]{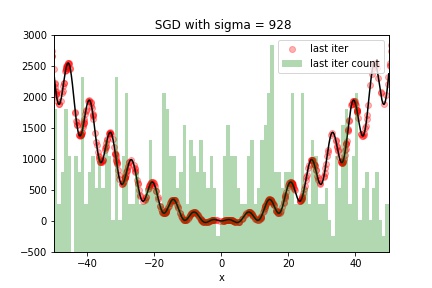}
         }
     \vspace{-2mm}
\caption{Comparison of Last iterate positions for SGD and Perturbed SGD without gradient noise for same noise levels. In each subfigure, $\zeta$ decides the noise level of both SGD and Perturbed SGD, as $\gamma$ is constant.\label{fig:add_exp}}
\end{figure}

From Figure~\ref{fig:add_exp}, we can see that SGD and Perturbed SGD have similar behaviour for low noise level, as the last iterates are able to escape local minima. But, if we keep increasing the noise level, SGD starts performing poorly and its last iterates get spread out evenly over the domain. In contrast, Perturbed SGD at the same noise level concentrates around the global minima, and only at the highest noise level of $\zeta = 20$, its last iterates start spreading out. Although SGD and Perturbed SGD are equal in expectation, there are key differences especially in high noise setting which motivates further investigation.

\end{document}